\documentclass[onefignum,onetabnum]{siamonline190516}

\usepackage{lipsum}
\usepackage{amsfonts}
\usepackage{graphicx}
\usepackage{epstopdf}
\usepackage{algorithmic}

\usepackage{type1cm}        
\usepackage{makeidx}         
\usepackage{graphicx}        
\usepackage{multicol}        
\usepackage[bottom]{footmisc}

\makeindex             

\usepackage{graphicx}
\usepackage{amsfonts,amsmath, amssymb, amstext,color}
\usepackage{graphicx} 

\usepackage{mathrsfs}
\usepackage{comment}
\usepackage{caption}
\usepackage{subcaption}

\usepackage{mathtools}
\DeclarePairedDelimiter\ceil{\lceil}{\rceil}
\DeclarePairedDelimiter\floor{\lfloor}{\rfloor}

\newcommand{\Wbf}{\mathbf{W}}

\newcommand{\KL}{{\rm KL}}
\newcommand{\ep}{\epsilon}
\newcommand{\Zbf}{\mathbf{Z}}

\newcommand{\Acal}{\mathcal{A}}

\newcommand{\Nbb}{\mathbb{N}}
\newcommand{\GP}{\mathrm{GP}}

\newcommand{\Xbf}{\mathbf{X}}
\newcommand{\Ybf}{\mathbf{Y}}

\newcommand{\myspan}{\mathrm{span}}

\newcommand{\Srm}{\mathrm{S}}

\newcommand{\OT}{\mathrm{OT}}
\newcommand{\Pcal}{\mathcal{P}}
\newcommand{\Probs}{\mathcal{P}}
\newcommand{\ADM}{\mathrm{Joint}}
\newcommand{\Gauss}{\mathrm{Gauss}}
\newcommand{\zbf}{\mathbf{z}}

\ifpdf
  \DeclareGraphicsExtensions{.eps,.pdf,.png,.jpg}
\else
  \DeclareGraphicsExtensions{.eps}
\fi

\usepackage{enumitem}
\setlist[enumerate]{leftmargin=.5in}
\setlist[itemize]{leftmargin=.5in}

\newsiamremark{remark}{Remark}
\crefname{hypothesis}{Hypothesis}{Hypotheses}
\newsiamthm{claim}{Claim}

\headers{Entropic Wasserstein distances between Gaussian processes}{H.Q. Minh}

\title{Finite sample approximations of exact and entropic Wasserstein distances between covariance operators and Gaussian processes\thanks{Submitted to the editors DATE.
\funding{This work was partially supported by JSPS KAKENHI Grant Number JP20H04250.}}}

\author{H\`a Quang Minh\thanks{RIKEN Center for Advanced Intelligence Project, Tokyo, Japan 
  (\email{minh.haquang@riken.jp})}}

\usepackage{amsopn}

\newcommand{\mapto}{\ensuremath{\rightarrow}}

\newcommand{\approach}{\ensuremath{\rightarrow}}
\newcommand{\imply}{\ensuremath{\Rightarrow}}

\newcommand{\equivalent}{\ensuremath{\Longleftrightarrow}}
\newcommand{\inclusion}{\ensuremath{\hookrightarrow}}

\newcommand{\R}{\mathbb{R}}

\newcommand{\bP}{\mathbb{P}}
\newcommand{\bE}{\mathbb{E}}

\newcommand{\Fcal}{\mathcal{F}}

\renewcommand{\S}{\mathcal{S}}

\newcommand{\la}{\langle}
\newcommand{\ra}{\rangle}

\renewcommand{\H}{\mathcal{H}}

\newcommand{\Lcal}{\mathcal{L}}

\def\trace{{\rm tr}}

\newcommand\HS{{\rm HS}}

\def\Sym{{\rm Sym}}

\def\tr{{\rm tr}}

\def\Tr{{\rm Tr}}

\newcommand{\z}{\mathbf{z}}

\def\1{\mathbf{1}}

\newcommand{\Ncal}{\mathcal{N}}

\newcommand{\Bsc}{\mathscr{B}}

\ifpdf
\hypersetup{
	pdftitle={An Example Article},
	pdfauthor={D. Doe, P. T. Frank, and J. E. Smith}
}
\fi

\usepackage{type1cm}        
\usepackage{makeidx}         
\usepackage{graphicx}        
\usepackage{multicol}        
\usepackage[bottom]{footmisc}

\makeindex             

\usepackage{graphicx}
\usepackage{amsfonts,amsmath, amssymb, amstext,color}
\usepackage{graphicx} 

\usepackage{mathrsfs}
\usepackage{comment}

\usepackage{mathtools}

\begin{document}
	\maketitle
		
		%
		%
		
		
		\begin{abstract}
			This work studies finite sample approximations of the exact and entropic regularized Wasserstein distances between centered Gaussian processes and, more generally, covariance operators of functional random processes. We first show that these distances/divergences are fully represented by
			reproducing kernel Hilbert space (RKHS) covariance and cross-covariance operators associated with the corresponding covariance functions.
			Using this representation, we show that the Sinkhorn divergence between two centered Gaussian processes can be consistently and efficiently estimated from the divergence between their corresponding normalized finite-dimensional covariance matrices, or alternatively, their sample covariance operators.
			Consequently, this leads to a consistent and efficient algorithm for estimating the Sinkhorn divergence 
			from finite samples generated by the two processes.  
			For a fixed regularization parameter, the convergence rates are {\it dimension-independent} and of the same order as those for the Hilbert-Schmidt distance.
			If at least one of the RKHS is finite-dimensional, we obtain a {\it dimension-dependent} sample complexity for the exact Wasserstein distance between the Gaussian processes. 
		\end{abstract}
		

\section{Introduction}
\label{section:introduction}

This work studies exact and entropic regularized Wasserstein distances and divergences between centered Gaussian processes, and more generally, between covariance operators 
associated with functional random processes.
Our main focus is on the 
finite sample approximations of the entropic divergences, which we show to be {\it dimension-independent}.
Our main results are obtained via the analysis of reproducing kernel Hilbert space (RKHS) covariance and cross-covariance operators associated with the covariance functions of the given random processes.
This work builds upon \cite{Minh2020:EntropicHilbert,Minh:2021EntropicConvergenceGaussianMeasures},
which formulated entropic Wasserstein distances between Gaussian measures on
 Hilbert spaces and their convergence properties.

The topic of distances/divergences between covariance operators and stochastic processes has attracted increasing interests in statistics and machine learning recently, e.g. \cite{Panaretos:jasa2010,Fremdt:2013testing,Pigoli:2014,masarotto2019procrustes,Matthews2016sparseKL,Sun2019functionalKL}. In \cite{Matthews2016sparseKL,Sun2019functionalKL}, the Kullback-Leibler divergence between stochastic processes was studied, the latter in the context of functional Bayesian neural networks.
In the field of functional data analysis,
 see e.g. \cite{Ramsay:2005functional,Ferraty:2006nonparametric,Horvath:2012functional},
one particular approach for analyzing functional data has been via the analysis of covariance operators. 
Recent work along this direction includes \cite{Panaretos:jasa2010,Fremdt:2013testing}, which utilize the Hilbert-Schmidt distance between covariance operators and \cite{Pigoli:2014, masarotto2019procrustes}, which utilize non-Euclidean distances, in particular the Procrustes distance, also known as Bures-Wasserstein distance.
The latter distance
is precisely the $2$-Wasserstein distance between two centered Gaussian measures on Hilbert space in the setting of optimal transport (OT) and can better capture the intrinsic geometry of the set of covariance operators. 
This distance is always well-defined for {\it singular covariance operators}, which is a distinct advantage over the Kullback-Leibler divergence, which requires equivalent Gaussian measures \cite{Minh:2020regularizedDiv}.
OT distances are, however, generally numerically difficult to compute and  can have, moreover, poor convergence rates (more below), which motivated
the study of entropic regularized OT. This direction has recently attracted much attention in 
machine learning, statistics, and related fields
\cite{cuturi13,
	feydy18,
	genevay17,
	mena2019samplecomplexityEntropicOT,
	patrini18}). This line of research is also closely connected
with the {\it Schr\"odinger bridge problem}~\cite{Schr31}, which has been studied extensively
~\cite{BorLewNus94,Csi75,peyre17,
	LeoSurvey,
	rus98}.

In \cite{Mallasto2020entropyregularized,Janati2020entropicOT,barrio2020entropic},
explicit formulas were obtained for
the entropic regularized 2-Wasserstein distance and Sinkhorn divergence between Gaussian measures on Euclidean space. 
These
were generalized
 to infinite-dimensional  
Gaussian measures on
Hilbert spaces in \cite{Minh2020:EntropicHilbert},
with the entropic formulation
being valid for both settings of {\it singular} and {\it nonsingular} covariance operators.
The 
Gaussian setting reveals explicitly several {\it favorable theoretical properties} of the entropic regularization formulation, including strict convexity, unique solution of barycenter equation in the singular setting, and Fr\'echet differentiability,
in contrast to the $2$-Wasserstein distance, which is 
{\it not} Fr\'echet differentiable in the infinite-dimensional setting.
%

Furthermore, it has been shown that
the Sinkhorn divergence has much {\it better convergence behavior} and {\it sample complexity} compared with the exact Wasserstein distance. 
%
It is well-known that
the sample complexity of the Wasserstein distance
can grow exponentially in the dimension of the underlying space $\R^d$, with the worst case being $O(n^{-1/d})$ \cite{dudley1969speed,
	weed17,fournier2015rate,horowitz1994mean}. 
In \cite{genevay18sample}, it is shown that,
as a consequence of entropic regularization,
the Sinkhorn divergence between two probability measures with {\it bounded support} on $\R^d$ achieves sample complexity
$O((1+\ep^{-\floor{d/2}})n^{-1/2})$, that is the same as the Maximum Mean Discrepancy (MMD) for a fixed $\ep>0$.
However, 
the constant factor in 
the sample complexity in \cite{genevay18sample} depends exponentially on the diameter of the support.
In \cite{mena2019samplecomplexityEntropicOT}, the rate of convergence 
$O\left(\ep\left(1+\frac{\sigma^{\ceil{5d/2} + 6}}{\ep^{\ceil{5d/4} + 3}}\right)n^{-1/2}\right)$
was obtained for $\sigma^2$-subgaussian measures on $\R^d$.
In \cite{Minh:2021EntropicConvergenceGaussianMeasures},
it was shown that the Sinkhorn divergence in the RKHS setting achieves the rate of convergence
$O\left((1+\frac{1}{\ep})n^{-1/2}\right)$ for all $\ep > 0$, which is thus {\it dimension-independent}.
In particular, this applies to Sinkhorn divergence between Gaussian measures on Euclidean space
 and infinite-dimensional Hilbert spaces.



{\bf Contributions of this work}.
In this work, 
we apply the results in \cite{Minh2020:EntropicHilbert,Minh:2021EntropicConvergenceGaussianMeasures} to the setting
of centered Gaussian processes, and more generally, covariance operators associated with functional random processes. Specifically,
\begin{enumerate}
	\item We show that the Wasserstein distance/Sinkhorn divergence
	between centered Gaussian processes are fully represented by RKHS covariance and cross-covariance operators associated with the corresponding covariance functions.
	 From this representation, we show that the Sinkhorn divergence 
	 can be consistently and efficiently estimated via the corresponding {\it normalized
	 finite covariance matrices}.
	 The convergence rate is {\it dimension-independent} and has the form $O(\frac{1}{\ep}n^{-1/2})$
	 (Section \ref{section:Sinkhorn-estimate-known-covariance-function}). Alternatively,
	 the Sinkhorn divergence can be consistently estimated via the corresponding {\it sample covariance operators} with similar convergence rate (Section \ref{section:Sinkhorn-estimate-sample-covariance}).
	\item We present an algorithm that consistently and efficiently estimates the Sinkhorn divergence 
	from 
	{\it finite samples} of the two given random processes. The convergence rate is {\it dimension-independent}.
	(Section \ref{section:Sinkhorn-estimate-unknown-covariance-functions}).
	\item For the exact $2$-Wasserstein distance, we obtain the corresponding sample complexity when
	the RKHS of at least one of the covariance functions is finite-dimensional.
	The convergence rate is {\it dimension-dependent} (Section \ref{section:exact-Wasserstein}).
\end{enumerate}


{\bf Notation}.
Throughout the paper,
let $\H$ be a real, separable Hilbert space, with $\dim(\H) = \infty$ unless explicitly stated otherwise.
Let $\Lcal(\H)$ denote the set of bounded linear operators on $\H$, 
with norm $||A|| = \sup_{||x||\leq 1}||Ax||$.
%
Let $\Sym(\H) \subset \Lcal(\H)$ be the set of bounded, self-adjoint linear operators on $\H$. Let $\Sym^{+}(\H) \subset \Sym(\H)$ be the set of
self-adjoint, {\it positive} operators on $\H$, i.e. $A \in \Sym^{+}(\H) \equivalent \la Ax,x\ra \geq 0 \forall x \in \H$. 
The Banach space $\Tr(\H)$  of trace class operators on $\H$ is defined by (e.g. \cite{ReedSimon:Functional})
$\Tr(\H) = \{A \in \Lcal(\H): ||A||_{\tr} = \sum_{k=1}^{\infty}\la e_k, (A^{*}A)^{1/2}e_k\ra < \infty\}$,
for any orthonormal basis $\{e_k\}_{k \in \Nbb} \in \H$, where $||\;||_{\tr}$ is the trace norm.
For $A \in \Tr(\H)$, its trace is then given by $\trace(A) = \sum_{k=1}^{\infty}\la e_k, Ae_k\ra$.
%
The Hilbert space $\HS(\H_1,\H_2)$ of Hilbert-Schmidt operators from $\H_1$ to $\H_2$ is defined by 
(e.g.\cite{Kadison:1983})
$\HS(\H_1, \H_2) = \{A \in \Lcal(\H_1, \H_2):||A||^2_{\HS} = \trace(A^{*}A) =\sum_{k=1}^{\infty}||Ae_k||^2 < \infty\}$,
for any orthonormal basis $\{e_k\}_{k \in \Nbb}$ in $\H_1$,
with inner product $\la A,B\ra_{\HS}=\trace(A^{*}B)$. For $\H_1 = \H_2 = \H$, we write $\HS(\H)$. 
We give more detail of $\HS(\H_1,\H_2)$ and the Hilbert-Schmidt norm $||\;||_{\HS}$ in Section \ref{section:Hilbert-Schmidt}.

{\it Proofs for all main results are presented in Section \ref{section:proofs}}.

\section{Background and previous work}
\label{section:background}
Let $(X,d)$ be a 
complete separable metric space equipped with a lower semi-continuous \emph{cost function} $c:X\times X \to \mathbb{R}_{\geq 0}$. 
Let $\Pcal(X)$ denote the set of all probability measures on $X$.
The {\it optimal transport} (OT) problem between two probability measures $\nu_0, \nu_1 \in \Probs(X)$ is  
(see e.g. \cite{villani2016})
\begin{equation}
	\OT(\nu_0, \nu_1) = \min_{\gamma\in \ADM(\nu_0,\nu_1)}\bE_\gamma[c] = \min_{\gamma \in \ADM(\nu_0, \nu_1)}\int_{X \times X}c(x,y)d\gamma(x,y)
	\label{equation:OT-exact}
\end{equation}
where $\ADM(\nu_0,\nu_1)$ is the set of joint probabilities with marginals $\nu_0$ and $\nu_1$.
%
For $1 \leq p < \infty$, let $\Pcal_p(X)$ denote the set of all probability measures $\mu$ on $X$ of finite moment of order $p$,
i.e. 
$\int_{X}d^p(x_0,x)d\mu(x) < \infty$ for some (and hence any) $x_0 \in X$.
The following {\it $p$-Wasserstein distance} $W_p$ between $\nu_0$ and $\nu_1$
defines a metric on $\Pcal_p(X)$ (Theorem 7.3, \cite{villani2016})
\begin{equation}
	W_p(\nu_0,\nu_1) = \OT_{d^p}(\nu_0, \nu_1)^{\frac{1}{p}}.
\end{equation}
For two Gaussian measures
$\nu_i=\Ncal(m_i,C_i)$, $i=0,1$ on $\R^n$ \cite{givens84,dowson82,olkin82,knott84} and  on a separable Hilbert space $\H$\cite{Gelbrich:1990Wasserstein,cuesta1996:WassersteinHilbert},
$W_2(\nu_0, \nu_1)$ admits the following closed form
\begin{equation}
	\label{equation:Gaussian-Wass-finite}
	W_2^2(\nu_0, \nu_1) = \|m_0-m_1\|^2 + \tr(C_0) + \tr(C_1) - 2 \tr\left(C_0^{1/2} C_1 C_0^{1/2}\right)^{1/2}.
\end{equation}
{\bf Entropic regularization.
}
The OT problem \eqref{equation:OT-exact} is often computationally challenging and it is more numerically efficient to solve the following regularized optimization problem \cite{cuturi13}
\begin{equation}
	\label{equation:OT-entropic}
	\OT_c^\epsilon(\mu, \nu) = \min_{\gamma\in \ADM(\mu,\nu)}\left\lbrace\bE_\gamma[c]
	+ \epsilon \KL(\gamma || \mu \otimes \nu) \right\rbrace, \;\; \ep > 0,
\end{equation}
where $\KL(\nu || \mu)$ denotes the Kullback-Leibler divergence between $\nu$ and $\mu$.
%
The KL
in \eqref{equation:OT-entropic} acts as a bias \cite{feydy18}, with the consequence that in general $\OT_c^\epsilon(\mu, \mu) \neq 0$. The following {\it $p$-Sinkhorn divergence} \cite{genevay17,feydy18} removes this bias
\begin{equation}
	\label{equation:sinkhorn}
	S_{p}^\epsilon(\mu, \nu) = \OT_{d^p}^\epsilon(\mu, \nu) - \frac{1}{2}(\OT_{d^p}^\epsilon(\mu,\mu) + \OT_{d^p}^\epsilon(\nu,\nu) ).
\end{equation}
In the case $X=\H$ is a separable Hilbert space and $\mu,\nu$ are Gaussian measures on $\H$,
both $\OT^{\ep}_{d^2}$ and $\Srm^{\ep}_{2}$ admit closed form expressions, as follows.
\begin{theorem}
	[\textbf{Entropic Wasserstein distance and Sinkhorn divergence between Gaussian measures on Hilbert space},
	\cite{Minh2020:EntropicHilbert}, Theorems 3, 4, and 7]
	\label{theorem:OT-regularized-Gaussian}
	Let 
	$\mu_0 = \Ncal(m_0, C_0)$, $\mu_1 = \Ncal(m_1, C_1)$ be two Gaussian measures on $\H$. For each fixed $\ep > 0$,
	\begin{align}
		\OT^{\ep}_{d^2}(\mu_0, \mu_1) &= ||m_0-m_1||^2 + \trace(C_0) + \trace(C_1) - \frac{\ep}{2}\trace(M^{\ep}_{01})
		+\frac{\ep}{2}\log\det\left(I + \frac{1}{2}M^{\ep}_{01}\right).
\\
			\Srm^{\ep}_{2}(\mu_0, \mu_1) &= ||m_0 - m_1||^2 + \frac{\ep}{4}\trace\left[M^{\ep}_{00} - 2M^{\ep}_{01} + M^{\ep}_{11}\right] 
				\label{equation:gauss-sinkhorn-infinite}
			\\
			& \quad
			 + \frac{\ep}{4}\log\det\left[\frac{\left(I + \frac{1}{2}M^{\ep}_{01}\right)^2}{\left(I + \frac{1}{2}M^{\ep}_{00}\right)\left(I + \frac{1}{2}M^{\ep}_{11}\right)}\right].
	\nonumber
	\end{align}
	The optimal joint measure is
	 the unique  Gaussian measure	
		$\gamma^{\ep} = \Ncal\left(\begin{pmatrix} m_0 \\ m_1 \end{pmatrix},
		\begin{pmatrix} C_0 & C_{XY}
			\\
			C_{XY}^{*} & C_1\end{pmatrix}
		\right)
		$,
		$\text{where  } C_{XY} = \frac{2}{\ep}C_0^{1/2}\left(I+\frac{1}{2}M^{\ep}_{01}\right)^{-1}
		C_0^{1/2}C_1$.
	Here $\det$ is the Fredholm determinant and
	$M^{\ep}_{ij}: \H \mapto \H$ 
	are trace class operators
	 defined by
		$M^{\ep}_{ij} = -I + \left(I + \frac{16}{\epsilon^2}C_i^{1/2}C_jC_i^{1/2}\right)^{1/2}$,
		$i,j=0,1$.
\end{theorem}
In particular, 
	$\lim_{\ep \approach 0}\OT^{\ep}_{d^2}(\mu_0, \mu_1) = \lim_{\ep \approach 0}\Srm^{\ep}_{2}(\mu_0,\mu_1) = W_2^2(\mu_0,\mu_1)$ and
	$\lim_{\ep \approach \infty}\Srm^{\ep}_{2}(\mu_0,\mu_1) = ||m_0 - m_1||^2$.
When $\dim(\H) < \infty$, we recover the finite-dimensional results in \cite{Mallasto2020entropyregularized,Janati2020entropicOT,barrio2020entropic}. 

	{\bf Convergence property}. $\Srm^{\ep}_{2}$ is a divergence function on $\Gauss(\H)$, the set of all Gaussian measures on $\H$ and has the following convergence property.
	
	

\begin{theorem}
	[\cite{Minh:2021EntropicConvergenceGaussianMeasures}-Theorems 2 and 5]
\label{theorem:Sinkhorn-convergence}
Let $A,B,\{A_n,B_n\}_{n \in \Nbb} \in \Sym^{+}(\H) \cap \Tr(\H)$. 
$\forall \ep > 0$,
\begin{align}
&\Srm^{\ep}[\Ncal(0, A_n), \Ncal(0,A)] \leq \frac{3}{\ep}[||A_n||_{\HS} + ||A||_{\HS}]||A_n - A||_{\HS}.
\\
&\left|\Srm^{\ep}_2[\Ncal(0,A_n), \Ncal(0,B_n)] - \Srm^{\ep}_2[\Ncal(0,A),\Ncal(0,B)]\right|
\nonumber
\\
&
\leq \frac{3}{\ep}[||A_n||_{\HS} + ||A||_{\HS} + 2||B||_{\HS}]||A_n - A||_{\HS}
\nonumber
\\
&\quad +\frac{3}{\ep}[2||A_n||_{\HS} + ||A||_{\HS} + ||B||_{\HS}]||B_n - B||_{\HS}.
\end{align}
\end{theorem}
In this work, we apply Theorems \ref{theorem:OT-regularized-Gaussian} and \ref{theorem:Sinkhorn-convergence}
to estimate Sinkhorn divergence between centered Gaussian processes and, more generally, covariance operators of random processes.

{\bf Related work}. 
The $2$-Wasserstein distance was applied to Gaussian processes in \cite{masarotto2019procrustes,Mallasto:NIPS2017Wasserstein},
however the treatment in \cite{Mallasto:NIPS2017Wasserstein} is generally only valid in finite dimensions.
Sample complexities were obtained in \cite{Minh:2021EntropicConvergenceGaussianMeasures} for Sinkhorn divergence between Gaussian measures on Euclidean and Hilbert spaces. In \cite{mallasto2021estimating}, the author obtained results similar to our Theorem
\ref{theorem:Sinkhorn-finite-sample-GaussianProcess},
however the main theoretical analysis carried out  in
\cite{mallasto2021estimating} is {flawed} (see discussion in Section \ref{section:Sinkhorn-estimate-known-covariance-function}).

\section{Kernels, covariance operators, and Gaussian processes}
\label{section:kernels-covariance}

Throughout the paper, we make the following assumptions
\begin{enumerate}
	\item {\bf Assumption 1} $T$ is a $\sigma$-compact metric space, that is
	$T = \cup_{i=1}^{\infty}T_i$, where $T_1 \subset T_2 \subset \cdots$, with each $T_i$ being compact.
	
	\item {\bf Assumption 2} $\nu$ is a non-degenerate Borel probability measure on $T$, that is $\nu(B) > 0$
	for each open set $B \subset T$.
	\item {\bf Assumption 3} $K, K^1, K^2: T \times T \mapto \R$ are continuous, symmetric, positive definite kernels
	and $\exists \kappa > 0, \kappa_1 > 0, \kappa_2 > 0$ such that
	\begin{align}
		\label{equation:Assumption-3}
	\int_{T}K(x,x)d\nu(x) \leq \kappa^2, \;\; \int_{T}K^i(x,x)d\nu(x) \leq \kappa_i^2.
	\end{align}
\item {\bf Assumption 4} $\xi\sim \GP(0, K)$, $\xi^i \sim \GP(0, K^i)$, $i=1,2$, are {\it centered} Gaussian processes
with covariance functions $K, K^i$, respectively.
\end{enumerate}
For $K$ satisfying Assumption 3, positivity implies $K(x,t)^2 \leq K(x,x)K(t,t)$ $\forall x,t \in T$, thus
\begin{align}
\int_{T}K(x,t)^2d\nu(t) < \infty \; \forall x \in T, \;\;\; \int_{T \times T}K(x,t)^2d\nu(x)d\nu(t) < \infty.
\end{align}
The first inequality means $K_x \in \Lcal^2(T,\nu)$ $\forall x \in T$, where $K_x:T \mapto \R$ is defined by $K_x(t) = K(x,t)$.
Let $\H_K$ denote the corresponding reproducing kernel Hilbert space (RKHS),
then $\H_K \subset \Lcal^2(T,\nu)$ \cite{
	Sun2005MercerNoncompact}. Define the following linear operator
\begin{align}
	&R_K = R_{K,\nu}: \Lcal^2(T, \nu) \mapto \H_K,
	\\
	&R_Kf = \int_{T}K_tf(t)d\nu(t), \;\;\; (R_Kf)(x) = \int_{T}K(x,t)f(t)d\nu(t).
	\label{equation:RK}
\end{align}
Since $||R_Kf||_{\H_K}\leq \int_{T}||K_t||_{\H_K}|f(t)|d\nu(t)\leq \sqrt{\int_{T}K(t,t)d\nu(t)}||f||_{\Lcal^2(T,\nu)}$,
$R_K$ is bounded, with
\begin{align}
||R_K:\Lcal^2(T,\nu) \mapto \H_K|| \leq \sqrt{\int_{T}K(t,t)d\nu(t)} \leq \kappa.
\end{align} 
Its adjoint is 
$R_K^{*}: \H_K \mapto \Lcal^2(T, \nu) = J:\H_K \inclusion \Lcal^2(T, \nu)$, the inclusion operator
from $\H_K$ into $\Lcal^2(T,\nu)$ \cite{Rosasco:IntegralOperatorsJMLR2010}. 
$R_K$ then
induces 
 the following self-adjoint, positive, compact operator
%
%
\begin{align}
\label{equation:CK}
C_K &= C_{K,\nu} = R_K^{*}R_K: \Lcal^2(T,\nu) \mapto \Lcal^2(T, \nu),
\\
(C_Kf)(x) & = \int_{T}K(x,t)f(t)d\nu(t), \;\; \forall f \in \Lcal^2(T,\nu),
\\
||C_K||_{\HS(\Lcal^2(T,\nu))}^2 & = \int_{T \times T}K(x,t)^2d\nu(x)d\nu(t) \leq \kappa^4.
\end{align}
The operator $C_K$
has been studied extensively,
e.g. \cite{CuckerSmale,Sun2005MercerNoncompact,Rosasco:IntegralOperatorsJMLR2010}.
Let $\{\lambda_k\}_{k \in \Nbb}$ be its eigenvalues
 with normalized eigenfunctions $\{\phi_k\}_{k\in \Nbb}$
forming an orthonormal basis in $\Lcal^2(T,\nu)$. A fundamental result for positive definite kernels is Mercer's Theorem, 
which states that
\begin{align}
K(x,y) = \sum_{k=1}^{\infty}\lambda_k \phi_k(x)\phi_k(y)\;\;\;\forall (x,y) \in T \times T,
\end{align}
(see version in \cite{Sun2005MercerNoncompact}), 
where the series converges absolutely for each pair $(x,y) \in T \times T$ and uniformly on any compact subset of $T$.
By Mercer's Theorem, $K$ is completely determined by $C_K$ and vice versa. The RKHS $\H_K$ is explicitly described by  
\begin{align}
\H_K = \left\{f \in \Lcal^2(T,\nu), f = \sum_{k=1}^{\infty}a_k \phi_k \;:\; ||f||^2_{\H_K} = \sum_{k=1, \lambda_k > 0}^{\infty}\frac{a_k^2}{\lambda_k} < \infty \right\} \subset \Lcal^2(T,\nu).
\end{align}
Furthermore, $C_K \in \Sym^{+}(\H) \cap \Tr(\H)$, $\H = \Lcal^2(T,\nu)$,
with
\begin{align}
\trace(C_K) &
= \sum_{k=1}^{\infty}\lambda_k = \int_{T}K(t,t)d\nu(t) \leq \kappa^2.
\end{align} 

{\bf Gaussian processes}.
Consider
the correspondence between
Gaussian measures, covariance operators $C_K$ as defined in Eq.\eqref{equation:CK}, and Gaussian processes with paths in $\Lcal^2(T,\nu)$
\cite{Rajput1972gaussianprocesses}.
Let $(\Omega, \Fcal, P)$ be a probability space, 
$\xi = (\xi(t))_{t \in T}
 = (\xi(\omega,t))_{t \in T}$ be 
a real
Gaussian process on $(\Omega, \Fcal,P)$, with mean $m$ and covariance function $K$, denoted by $\xi\sim \GP(m,K)$, where 
\begin{align}
m(t) = \bE{\xi(t)},\;\;K(s,t) = \bE[(\xi(s) - m(s))(\xi(t) - m(t))], \;\;s,t \in T.
\end{align}
The sample paths
$\xi(\omega,\cdot) \in \H = \Lcal^2(T, \nu)$ almost $P$-surely, i.e.
$\int_{T}\xi^2(\omega,t)d\nu(t) < \infty$ almost $P$-surely,
if and only if (\cite{Rajput1972gaussianprocesses}, Theorem 2 and Corollary 1)
\begin{align}
	\label{equation:condition-Gaussian-process-paths}
	\int_{T}m^2(t)d\nu(t) < \infty, \;\;\; \int_{T}K(t,t)d\nu(t) < \infty.
\end{align}
In this case, $\xi$ induces the following {\it Gaussian measure} $P_{\xi}$ on $(\H, \Bsc(\H))$:
$	P_{\xi}(B) = P\{\omega \in \Omega: \xi(\omega, \cdot) \in B\}, \; B \in \Bsc(\H)$,
with mean $m \in \H$ and covariance operator
$C_K: \H \mapto \H$, defined by Eq.\eqref{equation:CK}.
Conversely, let $\mu$ be a Gaussian measure on $(\H
, \Bsc(\H))$, then
there is a
Gaussian process $\xi = (\xi(t))_{t \in T}$
with sample paths in $\H$,
with induced probability measure $P_{\xi} = \mu$.

Since Gaussian processes are fully determined by their means and covariance functions,
the latter
being fully determined by their covariance operators, we can define distance/divergence functions between two Gaussian processes as follows, see also e.g. \cite{Panaretos:jasa2010,Fremdt:2013testing,Pigoli:2014,
	masarotto2019procrustes}.
\begin{definition}
	[\textbf{Divergence between Gaussian processes}]
	\label{definition:divergence-Gaussian-Process}
Assume Assumptions 1-4.
	Let $\H = \Lcal^2(T,\nu)$. 
	Let $\xi^i \sim \GP(m_i,K_i)$, $i=1,2$,  be
	two Gaussian processes with mean $m_i \in \H$ and covariance function $K^i$.
	Let $D$ be a divergence function on $\Gauss(\H)\times \Gauss(\H)$. The corresponding divergence
	$D_{\GP}$ 
	between $\xi^1$ and $\xi^2$ is defined to be
	\begin{align}
		D_{\GP}(\xi^1|| \xi^2) = D(\Ncal(m_1,C_{K^1}) ||\Ncal(m_2, C_{K^2})).
	\end{align}
\end{definition}
It is clear then that $D_{\GP}(\xi^1 ||\xi^2) \geq 0$ and
		$D_{\GP}(\xi^1 ||\xi^2) = 0 \equivalent m_1 = m_2, K^1 = K^2$.
		Subsequently, we assume $m_1  = m_2 = 0$ and compute
		$\Srm^{\ep}_2[\Ncal(0, C_{K^1}), \Ncal(0, C_{K^2})]$.



{\bf RKHS covariance operators}. 
To empirically estimate
$\Srm^{\ep}_2[\Ncal(0, C_{K^1}), \Ncal(0, C_{K^2})]$, we employ {\it RKHS covariance operators and cross-covariance operators}. The operator $R_K$ defined in Eq.\eqref{equation:RK} induces  the following self-adjoint, positive, compact {\it RKHS covariance operator}
\begin{align}
\label{equation:LK}
L_K &= R_KR_K^{*}: \H_K \mapto \H_K, \;\;L_K = \int_{T}(K_t \otimes K_t) d\nu(t),\;
\\
L_Kf(x) &= \int_{T}K_t(x)\la f, K_t\ra_{\H_K}d\nu(t) = \int_{T}K(x,t)f(t)d\nu(t),\;\; f\in \H_K.
\end{align}
$L_K$ has the same nonzero eigenvalues as $C_K$ and thus $L_K \in \Sym^{+}(\H_K) \cap \Tr(\H_K)$, with 
\begin{align}
 \trace(L_K) = \trace(C_K) \leq \kappa^2,\;\;||L_K||_{\HS(\H_K)} = ||C_K||_{\HS(\Lcal^2(T,\nu))} \leq \kappa^2.
\end{align}
We note also that for $f \in \H_K$, $C_Kf = L_Kf$. {\it However, as we see below, despite their many common properties, 
$C_K$ and $L_K$ are generally {\bf not} interchangeable}.

{\bf Empirical RKHS covariance operator}. Let $\Xbf = (x_i)_{i=1}^m$ be 
independently sampled from $T$ according to $\nu$. This defines the following empirical version of $L_K$
\begin{align}
L_{K,\Xbf} &= \frac{1}{m}\sum_{i=1}^m (K_{x_i} \otimes K_{x_i})
: \H_{K} \mapto \H_{K},
\\
L_{K,\Xbf}f  &= \frac{1}{m}\sum_{i=1}^m K_{x_i}\la f, K_{x_i}\ra_{\H_K} = \frac{1}{m}\sum_{i=1}^m f(x_i)K_{x_i}, \; f \in \H_K.
\end{align}
Let $\H_{K,\Xbf} = \myspan\{K_{x_i}\}_{i=1}^m \subset \H_K$,
 then
$L_{K,\Xbf}: \H_K \mapto \H_{K,\Xbf}$. In particular, $L_{K,\Xbf}:\H_{K,\Xbf} \mapto \H_{K,\Xbf}$ and
$\forall j, 1 \leq j \leq m$,
$L_{K,\Xbf}K_{x_j} = \frac{1}{m}\sum_{i=1}^mK(x_j, x_i)K_{x_i}$.
Let $K[\Xbf]$ denote the $m \times m$ {\it Gram matrix}, with $(K[\Xbf])_{ij} = K(x_i,x_j)$, then
the matrix representation of $L_{K,\Xbf}:\H_{K,\Xbf} \mapto \H_{K,\Xbf}$
in $\myspan\{K_{x_i}\}_{i=1}^m$ is $\frac{1}{m}K[\Xbf]$. In particular,
the nonzero eigenvalues of $L_{K,\Xbf}$ are precisely those of $\frac{1}{m}K[\Xbf]$, corresponding to 
eigenvectors that must lie in $\H_{K,\Xbf}$.
Thus, the nonzero eigenvalues of $C_K:\Lcal^2(T, \nu) \mapto \Lcal^2(T,\nu)$, $\trace(C_K)$, $||C_K||_{\HS}$,
which are the same as those of $L_K:\H_K \mapto \H_K$, can be empirically estimated from those of the $m \times m$ matrix $\frac{1}{m}K[\Xbf]$ (see \cite{Rosasco:IntegralOperatorsJMLR2010}).

{\bf RKHS cross-covariance operators}.
Let $K^1,K^2$ be two kernels satisfying Assumptions 1-4, and $\H_{K^1}, \H_{K^2}$ the corresponding
RKHS. Let $R_{K^i}:\Lcal^2(T,\nu) \mapto \H_{K^i}$, $i=1,2$ be as defined in Eq.\eqref{equation:RK}.
They give rise to the following {\it RKHS cross-covariance operators}
%
%
\begin{align}
	\label{equation:R12-operator}
	R_{12}&= R_{K^1}R_{K^2}^{*}: \H_{K^2} \mapto \H_{K^1},\;\;\;
	R_{21} = R_{K^2}R_{K^1}^{*}: \H_{K^1}\mapto \H_{K^2} = R_{12}^{*}.
\end{align}
Both $L_{K}$ and $R_{12},R_{21}$ are encompassed in 
the following, which is straightforward to verify.
\begin{lemma}
	\label{lemma:R12-theoretical}
	The operators $R_{ij} = R_{K^i}R_{K^j}^{*}:\H_{K^j} \mapto \H_{K^i}$, $i,j=1,2$,
	 are given by
	\begin{align}
		R_{ij} & = \int_{T}(K^i_t \otimes K^j_t)d\nu(t),\;\;
		R_{ij}f = \int_{T}K^i_t \la f, K^j_t\ra_{\H_{K^j}}d\nu(t), \;\; i,j=1,2, 
		\\
		R_{ij}f(x) &= \int_{T}K^i_t(x)f(t)d\nu(t) = \int_{T}K^i(x,t)f(t)d\nu(t),\;\; f\in \H_{K^j},
	\end{align}
Then $R_{ii} = L_{K^i}$, $R_{12}^{*} = R_{21}$, and the operator $R_{12}^{*}R_{12}:\H_{K^2} \mapto \H_{K^2}$ is given by
\begin{align}
	R_{12}^{*}R_{12}f 
	& = \int_{T}K^2_t\int_{T}K^1(t,u)f(u)d\nu(u)d\nu(t), \;\; f \in \H_{K^2},
	\\
	(R_{12}^{*}R_{12}f)(x) &= \int_{T \times T}K^2(x,t)K^1(t,u)f(u)d\nu(u)d\nu(t), \;\;\forall x \in T.
	\nonumber
\end{align}
\end{lemma}
We remark that with $f \in \Lcal^2(T,\nu)$, 
$C_{K^1}f(t) = \int_{T}K^1(t,u)f(u)d\nu(u)$ and
\begin{align}
	(C_{K^2}C_{K^1}f)(x) &= \int_{T}K^2(x,t)(C_{K^1}f)(t)d\nu(t) 
	=\int_{T\times T}K^2(x,t)K^1(t,u)f(u)d\nu(u)d\nu(t).
\end{align}
Thus for $f \in \H_{K^2}$, $C_{K^2}C_{K^1}f = R_{12}^{*}R_{12}f \in \H_{K^2}$, however
$C_{K^2}C_{K^1}: \Lcal^2(T,\nu) \mapto \Lcal^2(T,\nu)$ is generally {\it not} self-adjoint, whereas $R_{12}^{*}R_{12} 
\in \Sym^{+}(\H_{K^2})$.
\begin{lemma}
	[\textbf{Hilbert-Schmidt norm}]
	\label{lemma:R12-HS}
	Under Assumptions 1-3,
	$R_{ij} \in \HS(\H_{K^j}, \H_{K^i})$, with
	$||R_{ij}||_{\HS(\H_{K^j}, \H_{K^i})} \leq \kappa_i\kappa_j$, $i,j=1,2$.
\end{lemma}


\begin{lemma}
	[\textbf{Empirical RKHS covariance and cross-covariance operators}]
	\label{lemma:R12-empirical}
	Let $\Xbf = \{x_1, \ldots, x_m\}$ in $T^m$. Define the empirical integral operators
	$R_{ij,\Xbf}:\H_{K^j} \mapto \H_{K^i}$ $i,j=1,2$, by
	\begin{align}
		R_{ij,\Xbf} & = \frac{1}{m}\sum_{k=1}^m(K^i_{x_k} \otimes K^j_{x_k}): \H_{K^j} \mapto \H_{K^i},
		\\
		R_{ij,\Xbf}f &= \frac{1}{m}\sum_{k=1}^mK^i_{x_k}\la f, K^j_{x_k}\ra_{\H_{K^j}} = \frac{1}{m}\sum_{k=1}^mf(x_k)K^i_{x_k}, \; f\in \H_{K^j},
%
	\end{align}
	Then $R_{ii,\Xbf} = L_{K^i,\Xbf}$, $R_{12,\Xbf}^{*} = R_{21,\Xbf}$, and
	the operator $R_{12,\Xbf}^{*}R_{12,\Xbf}: \H_{K^2} \mapto \H_{K^2}$ is given by
	\begin{align}
		R_{12,\Xbf}^{*}R_{12,\Xbf}f = \frac{1}{m^2}\sum_{i,j=1}^mf(x_i)K^1(x_i, x_j)K^2_{x_j} \in \H_{K^2,\Xbf}.
	\end{align}
Thus $R_{12,\Xbf}^{*}R_{12,\Xbf}: \H_{K^2} \mapto \H_{K^2,\Xbf}$
	and on the subspace $\H_{K^2, \Xbf}$, in $\myspan\{K^2_{x_j}\}_{j=1}^m$, 
	$R_{12,\Xbf}^{*}R_{12,\Xbf}: \H_{K^2,\Xbf} \mapto \H_{K^2,\Xbf}$ has matrix representation 
	$\frac{1}{m^2}K^1[\Xbf]K^2[\Xbf]$.
\end{lemma}

\section{Estimation of Sinkhorn divergence from finite covariance matrices}
\label{section:Sinkhorn-estimate-known-covariance-function}

{\bf Main goal}. {\it Assume Assumptions 1-4. Our main goal in this work is to estimate
	$W_2[\Ncal(0, C_{K^1}), \Ncal(0, C_{K^2})]$ and
	$\Srm^{\ep}_2[\Ncal(0, C_{K^1}), \Ncal(0, C_{K^2})]$ given finite samples
	$\{\{\xi^1_i(x_j)\}_{i=1}^{N_1}, \{\xi^2_i(x_j)\}_{i=1}^{N_2}\}_{j=1}^m$
	from the two processes $\xi^1,\xi^2$ on 
	the set of points $\Xbf=(x_j)_{j=1}^m$ in $T$.
These correspond to $N_i$ realizations of process $\xi^i$, $i=1,2$, sampled at the $m$ points in $T$ given by $\Xbf$.}

Let $\Xbf = (x_i)_{i=1}^m$ be independently sampled from $(T,\nu)$. 
The Gaussian process assumption $\xi^i \sim \GP(0,K^i)$ means that $(\xi^i(., x_j))_{j=1}^m$ are $m$-dimensional Gaussian random variables,
with $(\xi^i(., x_j))_{j=1}^m \sim \Ncal(0, K^i[\Xbf])$, where
$(K^i[\Xbf])_{jk} = K^i(x_j, x_k)$, $1 \leq j,k \leq m$.
We first assume that the covariance matrices $K^i[\Xbf]$ are {\it known}. 
%
%
In this section, we show that 
\begin{align*}
\Srm^{\ep}_2\left[\Ncal\left(0, \frac{1}{m}K^1[\Xbf]\right), \Ncal\left(0, \frac{1}{m}K^2[\Xbf]\right)\right]
\;\text{\it consistently estimates}\;
\Srm^{\ep}_2[\Ncal(0, C_{K^1}), \Ncal(0,C_{K^2})].
\end{align*}

Let $\H$ be any separable Hilbert space. Let $c \in \R, c \neq 0$ be fixed. Consider the following function $G:\Sym^{+}(\H) \cap \Tr(\H)\mapto \R$
defined by
\begin{align}
\label{equation:function-G}
G(A) &= \trace[M(A)] - \log\det\left(I + \frac{1}{2}M(A)\right),\;\;
\text{where } M(A) = -I + (I+ c^2A)^{1/2}.
\end{align}
With this definition, with $c = \frac{4}{\ep}$, $\Srm^{\ep}_2[\Ncal(0, C_{K^1}), \Ncal(0, C_{K^2})]$ can be expressed as 
\begin{align}
	\label{equation:Sinkhorn-expression-1}
	\Srm^{\ep}_2[\Ncal(0, C_{K^1}), \Ncal(0, C_{K^2})] = \frac{1}{c}\left[G(C_{K^1}^2) + G(C_{K^2}^2) - 2G(C_{K^1}^{1/2}C_{K^2}C_{K^1}^{1/2})\right].
\end{align}
We now represent this via 
the RKHS covariance and cross-covariance operators
in Section 
\ref{section:kernels-covariance}.

{\bf RKHS covariance operator terms}.
Since $C_{K^i}\in\Tr(\Lcal^2(T,\nu))$ and $L_{K^i} \in \Tr(\H_{K^i})$, $i=1,2$, have the same nonzero eigenvalues,
we have $G(C_{K^i}^2) = G(L_{K^i}^2)$,
which can be approximated by their
empirical versions $G(L_{K^i,\Xbf}^2)$, $i=1,2$, which are the same as $G(\frac{1}{m^2}(K^i[\Xbf])^2)$.


{\bf RKHS cross-covariance operator term}.
Consider the term $G(C_{K^1}^{1/2}C_{K^2}C_{K^1}^{1/2})$ in Eq.
\eqref{equation:Sinkhorn-expression-1}.
Recall that 
$R_{K^i}: \Lcal^2(T, \nu) \mapto \H_{K^i}$,
$R_{K^i}^{*}:\H_{K^i} \mapto \Lcal^2(T,\nu)$,
with $C_{K^i} = R_{K^i}^{*}R_{K^i}:\Lcal^2(T,\nu)\mapto \Lcal^2(T,\nu)$. 
The nonzero eigenvalues of $C_{K^1}^{1/2}C_{K^2}C_{K^1}^{1/2}$ are the same as those of
\begin{align}
	C_{K^1}C_{K^2} = (R_{K^1}^{*}R_{K^1})(R_{K^2}^{*}R_{K^2}): \Lcal^2(T, \nu) \mapto \Lcal^2(T, \nu),
\end{align}
which, in turns, are the same as the nonzero eigenvalues of the operator 
\begin{align}
	&R_{12}^{*}R_{12} = R_{K^2}(R_{K^1}^{*}R_{K^1})R_{K^2}^{*}	: \H_{K^2} \mapto \H_{K^2},
\end{align}
or equivalently, of $R_{21}^{*}R_{21} = R_{K^1}(R_{K^2}^{*}R_{K^2})R_{K^1}^{*}: \H_{K^1} \mapto \H_{K^1}$.
Thus $G(C_{K^1}^{1/2}C_{K_2}C_{K_1}^{1/2}) = G(R_{12}^{*}R_{12})$, with the empirical version being 
$G(R_{12,\Xbf}^{*}R_{12,\Xbf})$.
By Lemma \ref{lemma:R12-empirical}, the nonzero eigenvalues of
$R_{12,\Xbf}^{*}R_{12,\Xbf}$
are 
those of 
$\frac{1}{m^2}K^1[\Xbf]K^2[\Xbf]$, or equivalently, of 
$\frac{1}{m^2}(K^{1}[\Xbf])^{1/2}\times$ $K^2[\Xbf](K^{1}[\Xbf])^{1/2}$. Thus $G(R_{12,\Xbf}^{*}R_{12,\Xbf}) = G(\frac{1}{m^2}(K^{1}[\Xbf])^{1/2}K^2[\Xbf]K^{1}([\Xbf])^{1/2})$. We thus have

\begin{proposition}
	[\textbf{RKHS covariance and cross-covariance operator representation for Sinkhorn divergence}]
	\label{propopsition:Sinkhorn-cross-covariance}
	Under Assumptions 1-4, let $\Xbf = (x_i)_{i=1}^m \in T^m$. Then
\begin{align}
&\Srm^{\ep}_2[\Ncal(0, C_{K^1}), \Ncal(0, C_{K^2})] = \frac{1}{c}\left[G({L_{K^1}^2}) + G(L_{K^2}^2) - 2 G(R_{12}^{*}R_{12})\right].
\label{equation:Sinkhorn-cross-cov-theoretical}
\\
&\Srm^{\ep}_2\left[\Ncal\left(0, \frac{1}{m}K^1[\Xbf]\right), \Ncal\left(0, \frac{1}{m}K^2[\Xbf]\right)\right]
\nonumber
\\
&\quad 
= \frac{1}{c}\left[G({L_{K^1,\Xbf}^2}) + G(L_{K^2,\Xbf}^2) - 2 G(R_{12,\Xbf}^{*}R_{12,\Xbf})\right].
	\label{equation:Sinkhorn-cross-cov-empirical}
\end{align}
\end{proposition}

The representations in Proposition \ref{propopsition:Sinkhorn-cross-covariance} suggest
that, given a random sample $\Xbf = (x_i)_{i=1}^m$, $\Srm^{\ep}_2\left[\Ncal\left(0, \frac{1}{m}K^1[\Xbf]\right), \Ncal\left(0, \frac{1}{m}K^2[\Xbf]\right)\right] \approach
\Srm^{\ep}_2[\Ncal(0, C_{K^1}), \Ncal(0, C_{K^2})]$ as $m \approach \infty$.
We now analyze the rate of this convergence.
The function $G$ as defined in Eq.\eqref{equation:function-G} satisfies the following
\begin{proposition}
	\label{proposition:continuity-trace-norm-G}
	Let $\H,\H_1, \H_2$ be separable Hilbert spaces. Then
	\begin{align}
		|G(A) - G(B)| &\leq \frac{3c^2}{4}||A-B||_{\tr} \;\;\;\forall A,B \in \Sym^{+}(\H)\cap \Tr(\H).
\\
		|G(A^2) - G(B^2)| &\leq \frac{3c^2}{4}[||A||_{\HS} + ||B||_{\HS}]||A-B||_{\HS}, \;\;\;\forall A,B \in \Sym(\H) \cap \HS(\H).
		\label{equation:G-A2}
	\end{align}
	Let $A,B \in \HS(\H_1,\H_2)$, then
	$A^{*}A, B^{*}B \in \Sym^{+}(\H_1) \cap \Tr(\H_1)$ and
	\begin{align}
		\label{equation:G-AstarA}
		|G(A^{*}A) - G(B^{*}B)| \leq \frac{3c^2}{4}[||A||_{\HS(\H_1, \H_2)} + ||B||_{\HS(\H_1, \H_2)}]||A-B||_{\HS(\H_1,\H_2)}. 
	\end{align}
\end{proposition}

By Proposition \ref{proposition:continuity-trace-norm-G},
we thus need to estimate $||L_{K^i,\Xbf}-L_{K^i}||_{\HS(\H^i)}$, $i=1,2$ and $||R_{12,\Xbf}-R_{12}||_{\HS(\H_{K^2}, \H_{K^1})}$.
We apply the following law of large numbers for Hilbert space-valued random variables, which
is a consequence of a general result by Pinelis (\cite{Pinelis1994optimum}, Theorem 3.4).
\begin{proposition}
	[\cite{SmaleZhou2007}]
	\label{proposition:Pinelis}
	Let $(Z,\Acal, \rho)$ be a probability space and
	$\xi:(Z,\rho) \mapto \H$ be a random variable.
	Assume that $\exists M > 0$ such that $||\xi|| \leq M < \infty$ almost surely.
	Let $\sigma^2(\xi)= \bE||\xi||^2$. Let $(z_i)_{i=1}^m$ be independently sampled 
	from $(Z, \rho)$.
	$\forall 0 < \delta < 1$, with probability at least $1-\delta$,
	\begin{align}
		\left\|\frac{1}{m}\sum_{i=1}^m\xi(z_i) - \bE\xi \right\| \leq \frac{2M\log\frac{2}{\delta}}{m} + \sqrt{\frac{2\sigma^2(\xi)\log\frac{2}{\delta}}{m}}.
	\end{align}
\end{proposition}

\subsection{Estimation with bounded kernels}
\label{section:bounded-kernels}
We first consider the following setting

{\bf Assumption 5}. $K, K^1, K^2$ are {\it bounded}, i.e.
$\exists \kappa, \kappa_1, \kappa_2 > 0$ such that
\begin{align}
\sup_{x \in T}K(x,x) \leq \kappa^2, \;\; \sup_{x \in T}K^i(x,x) \leq \kappa_i^2, i=1,2.
\end{align}
This is satisfied for exponential kernels $\exp(-a||x-y||^p)$, $a>0, 0 < p \leq 2$, on $T = \R^d$ and for all continuous kernels if $T$ is compact.
Applying Proposition \ref{proposition:Pinelis}, we obtain the following.

\begin{proposition}
	[\textbf{Convergence of RKHS empirical covariance and cross-covariance operators}]
	\label{proposition:concentration-TK2K1-empirical}
	Under Assumptions 1-5,
	$||R_{ij,\Xbf}||_{\HS(\H_{K^j}, \H_{K^i})} \leq \kappa_i \kappa_j$, $i,j=1,2$, $\forall \Xbf \in T^m$.
	Let $\Xbf = (x_i)_{i=1}^m$ be independently sampled from $(T,\nu)$.
	$\forall 0 < \delta < 1$, with probability at least $1-\delta$,
	\begin{align}
		||R_{ij,\Xbf} - R_{ij}||_{\HS(\H_{K^j}, \H_{K^i})} \leq \kappa_i\kappa_j\left[ \frac{2\log\frac{2}{\delta}}{m} + \sqrt{\frac{2\log\frac{2}{\delta}}{m}}\right],
		\\
		||R_{ij,\Xbf}^{*}R_{ij,\Xbf} - R_{ij}^{*}R_{ij}||_{\tr(\H_{K^j})} \leq 2\kappa_i^2\kappa_j^2\left[ \frac{2\log\frac{2}{\delta}}{m} + \sqrt{\frac{2\log\frac{2}{\delta}}{m}}\right].
	\end{align}	
In particular, 
$||L_{K^i,\Xbf}||_{\HS(\H_{K^i})} \leq \kappa_i^2$, 
and $\forall 0 <\delta <1$, 
with probability at least $1-\delta$,
\begin{align}
	\label{equation:LKX-LK-bounded}
	\left\|L_{K^i,\Xbf} - L_{K^i}\right\|_{\HS(\H_{K^i})} \leq \kappa_i^2\left(\frac{2\log\frac{2}{\delta}}{m} + \sqrt{\frac{2\log\frac{2}{\delta}}{m}}\right). 
\end{align}
	\end{proposition}
Proposition \ref{proposition:concentration-TK2K1-empirical} generalizes
Proposition 4 in \cite{Minh:2021EntropicConvergenceGaussianMeasures},
which states the bound in Eq.\eqref{equation:LKX-LK-bounded}.
Combining Propositions \ref{propopsition:Sinkhorn-cross-covariance}, \ref{proposition:continuity-trace-norm-G},
and \ref{proposition:concentration-TK2K1-empirical},
 we are led to our first main result.
\begin{theorem}
	[\textbf{Estimation of Sinkhorn divergence between Gaussian processes from finite covariance matrices - bounded kernels}]
	\label{theorem:Sinkhorn-finite-sample-GaussianProcess}
	Assume Assumptions 1-5. Let $\Xbf = (x_i)_{i=1}^m$ be independently sampled from $(T,\nu)$.
	For any $0< \delta < 1$, with probability at least $1-\delta$,
	\begin{align}
		&\left|\Srm^{\ep}_{2}\left[\Ncal\left(0, \frac{1}{m}K^1[\Xbf]\right), \Ncal\left(0, \frac{1}{m}K^2[\Xbf]\right) - \Srm^{\ep}_{2}\left[\Ncal(0,C_{K^1}), \Ncal(0,C_{K^2} )\right]\right]\right| 
		\nonumber
		\\
		&\quad \leq \frac{6}{\ep}(\kappa_1^2 + \kappa_2^2)^2\left[ \frac{2\log\frac{6}{\delta}}{m} + \sqrt{\frac{2\log\frac{6}{\delta}}{m}}\right].
	\end{align}
\end{theorem}
{\bf Discussion}. 
Similar results to Theorem \ref{theorem:Sinkhorn-finite-sample-GaussianProcess} were reported in 
(\cite{mallasto2021estimating}, Theorem 8). However, the main theoretical analysis in \cite{mallasto2021estimating} is {flawed}, in particular Proposition 2 there. We note that  
in the last term of Eq.\eqref{equation:Sinkhorn-expression-1}, $C_{K^i}\in \Tr(\Lcal^2(T,\nu))$ {\it cannot} be replaced by $L_{K^i} \in \Tr(\H_{K^i})$.
This is because in general $\H_{K^1}$ and $\H_{K^2}$ are {\it different}.
For example, let $T\subset \R^n$ be compact with nonempty interior, $K^1$ be the Gaussian kernel
$K(x,y) = \exp(-\sigma ||x-y||^2), \sigma > 0$, and $K^2$ be any polynomial kernel. 
Then $\H_{K^1}$ does {\it not} contain any polynomial \cite{MinhGaussian2010}, that is $\H_{K^1}\cap \H_{K^2} = \{0\}$.
Thus, while $C_{K^1}C_{K^2}$ and $(C_{K^1})^{1/2}C_{K^2}(C_{K^1})^{1/2}$ are well-defined on $\Lcal^2(T,\nu)$,  the products $L_{K^1}^{1/2}L_{K^2}L_{K^1}^{1/2}$ and $L_{K^1}L_{K^2}$ are generally {\it not} defined.

Furthermore, for any compact operators $A_i \in \Lcal(\H)$, $i=1,2$, the operators $A_i^{*}A$ and
$A_iA_i^{*}$ have the same nonzero eigenvalues, but this is generally {\it not} true
for the products $A_1^{*}A_1A_2^{*}A_2$ and $A_1A_1^{*}A_2A_2^{*}$, which generally have different nonzero eigenvalues. This can be readily verified numerically when $\H=\R^n$.

Thus (\cite{mallasto2021estimating}, Proposition 2) and the most crucial step in its proof, which substitutes $C_{K^i}$ by $L_{K^i}$, is {\it not} valid. 
We note also that Kato's Theorem (\cite{mallasto2021estimating}, Theorem 1) does {not} apply to 
operators of the form $AB$, $A,B\in \Sym^{+}(\H)$ compact, since $AB$ is generally {not} self-adjoint.

\subsection{Estimation with general kernels}
\label{section:general-kernels}
Consider the following more general setting, where the kernels are {\it not} necessarily bounded,
e.g. polynomial kernels on $T = \R^d$, $d \in \Nbb$. In this case, the sample bounds are less tight 
compared to those in Section \ref{section:bounded-kernels}.

{\bf Assumption 6}
$\exists \kappa, \kappa_1, \kappa_2 > 0$ such that
\begin{align}
	\int_{T}[K(x,x)]^2d\nu(x) \leq \kappa^4, \;\; \int_{T}[K^i(x,x)]^2d\nu(x) \leq \kappa_i^4, \; i=1,2.
\end{align}
This is related to the fourth moments of the processes (Assumption 6(*) in Section \ref{section:general-stochastic-process}) and implies \eqref{equation:Assumption-3} in Assumption 3.
The following is the corresponding version of Proposition \ref{proposition:concentration-TK2K1-empirical}.
\begin{proposition}
	\label{proposition:concentration-TK2K1-empirical-unbounded}
	Under Assumptions 1-4 and 6,
	let $\Xbf = (x_i)_{i=1}^m$ be independently sampled from $(T,\nu)$.
	For any $0 <\delta <1$, with probability at least $1-\delta$,	
	\begin{align}
		||R_{ij,\Xbf}||_{\HS(\H_{K^j}, \H_{K^i})} \leq \frac{2\kappa_i\kappa_j}{\delta},\;\;\; ||R_{ij,\Xbf}- R_{ij}||_{\HS(\H_{K^j}, \H_{K^i})}\leq \frac{2\kappa_i\kappa_j}{\sqrt{m}\delta}.
	\end{align}
In particular, for $K^i = K^j = K$,  $\forall 0 <\delta <1$, with probability at least $1-\delta$,
\begin{align}
	||L_{K,\Xbf}||_{\HS(\H_K)} \leq \frac{2\kappa^2}{\delta}, \;\;\; ||L_{K,\Xbf} - L_K||_{\HS(H_K)} \leq \frac{2\kappa^2}{\sqrt{m}\delta}.
\end{align}
\end{proposition}
Combining Propositions \ref{proposition:continuity-trace-norm-G} and \ref{proposition:concentration-TK2K1-empirical-unbounded},
we obtain the corresponding version of Theorem \ref{theorem:Sinkhorn-finite-sample-GaussianProcess}.
\begin{theorem}
	[\textbf{Estimation of Sinkhorn divergence between Gaussian processes from finite covariance matrices - general kernels}]
	\label{theorem:Sinkhorn-finite-sample-GaussianProcess-unbounded}
	Under Assumptions 1-4 and 6, let $\Xbf = (x_i)_{i=1}^m$ be independently sampled from $(T,\nu)$.
	For any $0 <\delta <1$, with probability at least $1-\delta$,
	\begin{align}
		&\left|\Srm^{\ep}_2\left[\Ncal\left(0, \frac{1}{m}K^1[\Xbf]\right), \Ncal\left(0, \frac{1}{m}K^2[\Xbf]\right)\right]- \Srm^{\ep}_2[\Ncal(0,C_{K^1}), \Ncal(0, C_{K^2})] \right|
		\nonumber
		\\
		& \quad \leq \frac{18}{\ep}(\kappa_1^2+\kappa_2^2)^2\left(1+\frac{6}{\delta}\right)\frac{1}{\sqrt{m}\delta}.
	\end{align}
\end{theorem}

\section{Estimation of Sinkhorn divergence from finite samples}
\label{section:Sinkhorn-estimate-unknown-covariance-functions}

We now return to our main goal stated in Section \ref{section:Sinkhorn-estimate-known-covariance-function},
where we are only given finite samples of the Gaussian processes $\xi^1,\xi^2$.
It is then necessary to estimate the covariance matrices $K^1[\Xbf], K^2[\Xbf]$
and the Sinkhorn divergence between them.
{\it For simplicity and without loss of generality, in the theoretical analysis we let $N_1 = N_2 = N$}.
We recall that the Gaussian process $\xi = (\xi(\omega, t))$ is defined on the probability space
$(\Omega, \Fcal, P)$. Let $\Wbf = (\omega_1, \ldots, \omega_N)$ be independently sampled from $(\Omega,P)$, which corresponds to $N$ sample paths
$\xi_i(x) = \xi(\omega_i,x), 1 \leq i \leq N, x \in T$.
Let $\Xbf =(x_i)_{i=1}^m \in T^m$ be fixed. 
Consider the following $m \times N$ data matrix
\begin{align}
\Zbf = \begin{pmatrix}
	\xi(\omega_1, x_1), \ldots, \xi(\omega_N, x_1),
	\\
	\cdots  
	\\
	\xi(\omega_1, x_m), \ldots, \xi(\omega_N, x_m)
\end{pmatrix}
= [\zbf(\omega_1), \ldots \zbf(\omega_N)] \in \R^{m \times N}.
\end{align}
Here
$\zbf(\omega) = (\z_i(\omega))_{i=1}^m = (\xi(\omega, x_i))_{i=1}^m$.
Since $(K[\Xbf])_{ij} = \bE[\xi(\omega,x_i)\xi(\omega, x_j)]$, $1\leq i,j\leq m$,
\begin{align}
	\label{equation:K-exact-W}
K[\Xbf] = \bE[\zbf(\omega)\zbf(\omega)^T] = \int_{\Omega}\zbf(\omega)\zbf(\omega)^TdP(\omega).
\end{align}
The empirical version of $K[\Xbf]$, using the random sample $\Wbf = (\omega_i)_{i=1}^N$, is then
\begin{align}
	\label{equation:K-hat-W}
\hat{K}_{\Wbf}[\Xbf] = \frac{1}{N}\sum_{i=1}^N \zbf(\omega_i)\zbf(\omega_i)^T = \frac{1}{N}\Zbf\Zbf^T.
\end{align}

The convergence of $\hat{K}_{\Wbf}[\Xbf]$ to $K[\Xbf]$ is given by the following.
\begin{proposition}
	\label{proposition:concentration-empirical-covariance}
Assume Assumptions 1-5.
Let $\xi \sim \GP(0,K)$ 
on $(\Omega, \Fcal,P)$
Let $\Xbf = (x_i)_{i=1}^m \in T^m$ be fixed. Then
$||K[\Xbf]||_F \leq m\kappa^2$.
Let $\Wbf = (\omega_1, \ldots, \omega_N)$ be independently sampled from $(\Omega,P)$. 
For any $0 < \delta < 1$, with probability at least $1-\delta$, 
\begin{align}
	&||\hat{K}_{\Wbf}[\Xbf] - K[\Xbf]||_F \leq \frac{2\sqrt{3}m\kappa^2}{\sqrt{N}\delta},
	\;\;
	||\hat{K}_{\Wbf}[\Xbf]||_F \leq \frac{2m\kappa^2}{\delta}.
\end{align}
\end{proposition}

Let now $\xi^i \sim \GP(0, K^i)$, $i=1,2$, 
on the probability spaces $(\Omega_i, \Fcal_i, P_i)$, respectively.
Let $\Wbf^i = (\omega^i_j)_{j=1}^N$, be independently sampled from $(\Omega_i, P_i)$, corresponding to the sample paths
$\{\xi^i_j(t) = \xi^i(\omega_j,t)\}_{j=1}^N$, $t \in T$, from $\xi^i$, $i=1,2$.
Combining Proposition \ref{proposition:concentration-empirical-covariance} and
Theorem \ref{theorem:Sinkhorn-convergence}, we obtain the following 
empirical estimate of 
$\Srm^{\ep}_2\left[\Ncal\left(0, \frac{1}{m}K^1[\Xbf]\right), \Ncal\left(0, \frac{1}{m}K^2[\Xbf]\right)\right]$
from two finite samples  of 
$\xi^1\sim \GP(0, K^1)$ and $\xi^2\sim \GP(0,K^2)$ given 
by $\Wbf^1, \Wbf^2$.
\begin{theorem}
	\label{theorem:Sinkhorn-estimate-unknown-1}
	Assume Assumptions 1-5.
Let $\Xbf = (x_i)_{i=1}^m \in T^m$, $m \in \Nbb$ be fixed.
Let $\Wbf^1 = (\omega_j^1)_{j=1}^N$, $\Wbf^2 = (\omega_j^2)_{j=1}^N$ be independently sampled from
$(\Omega_1, P_1)$ and $(\Omega_2, P_2)$, respectively.
For any $0 < \delta<1$, with probability at least $1-\delta$,
{\small
\begin{align}
&\left|\Srm^{\ep}_2\left[\Ncal\left(0, \frac{1}{m}\hat{K}^1_{\Wbf^1}[\Xbf]\right), \Ncal\left(0, \frac{1}{m}\hat{K}^2_{\Wbf^2}[\Xbf]\right)\right]
- \Srm^{\ep}_2\left[\Ncal\left(0, \frac{1}{m}K^1[\Xbf]\right), \Ncal\left(0, \frac{1}{m}K^2[\Xbf]\right)\right]
\right|
\nonumber
\\
& \leq
\frac{12\sqrt{3}}{\ep \delta}\left[\left(1+\frac{4}{\delta}\right)\kappa_1^4 + \left(3 + \frac{8}{\delta}\right)\kappa_1^2\kappa_2^2 + \kappa_2^4\right]\frac{1}{\sqrt{N}}. 
\end{align}
}
Here the probability is with respect to the product space $(\Omega_1,P_1)^N \times (\Omega_2,P_2)^N$.
\end{theorem}

Combing Theorem \ref{theorem:Sinkhorn-estimate-unknown-1} with Theorem \ref{theorem:Sinkhorn-finite-sample-GaussianProcess}, we are finally led to the following empirical estimate
of the theoretical Sinkhorn divergence $\Srm^{\ep}_2[\Ncal(0,C_{K^1}), \Ncal(0,C_{K^2})]$
from two finite samples $\Zbf^1,\Zbf^2$ of $\xi^1 \sim \GP(0,K^1)$ and $\xi^2 \sim \GP(0,K^2)$.

\begin{theorem}
	[\textbf{Estimation of Sinkhorn divergence between Gaussian processes from finite samples - bounded kernels}]
	\label{theorem:Sinkhorn-estimate-unknown-2}
	Assume Assumptions 1-5.
	Let $\Xbf = (x_i)_{i=1}^m$
	be independently sampled from $(T, \nu)$.
	Let $\Wbf^1 = (\omega_j^1)_{j=1}^N$, $\Wbf^2 = (\omega_j^2)_{j=1}^N$ be independently sampled from
	$(\Omega_1, P_1)$ and $(\Omega_2, P_2)$, respectively.
For any $0 < \delta<1$, with probability at least $1-\delta$,
{\small
\begin{align}
&\left|\Srm^{\ep}_2\left[\Ncal\left(0, \frac{1}{m}\hat{K}^1_{\Wbf^1}[\Xbf]\right), \Ncal\left(0, \frac{1}{m}\hat{K}^2_{\Wbf^2}[\Xbf]\right)\right] - \Srm^{\ep}_2[\Ncal(0,C_{K^1}), \Ncal(0,C_{K^2})]\right|
\nonumber
\\
&\leq
\frac{6}{\ep}(\kappa_1^2 + \kappa_2^2)^2\left[ \frac{2\log\frac{12}{\delta}}{m} + \sqrt{\frac{2\log\frac{12}{\delta}}{m}}\right]
+\frac{24\sqrt{3}}{\ep \delta}\left[\left(1+\frac{8}{\delta}\right)\kappa_1^4 + \left(3 + \frac{16}{\delta}\right)\kappa_1^2\kappa_2^2 + \kappa_2^4\right]\frac{1}{\sqrt{N}}.
\end{align}
}
Here the probability is with respect to the space $(T,\nu)^m \times (\Omega_1,P_1)^N \times (\Omega_2,P_2)^N$.
\end{theorem}
We note that if $\kappa_1,\kappa_2$ are absolute constants, e.g. for exponential kernels,
then the convergence rate in Theorem \ref{theorem:Sinkhorn-estimate-unknown-2} is {\it completely dimension-independent}.
Theorem \ref{theorem:Sinkhorn-estimate-unknown-2} provides the
theoretical justification for the Sinkhorn divergence estimation in Algorithm \ref{algorithm:estimate-Sinkhorn} (we set $N_1 = N_2 = N$ in the theoretical analysis only for simplicity).
\begin{algorithm}
	\caption{Estimate Wasserstein distance and Sinkhorn divergence between centered Gaussian processes
		from finite samples}
	\label{algorithm:estimate-Sinkhorn}
	\textbf{Input}: Finite samples $\{\xi^i_k(x_j)\}$,
	from $N_i$ realizations $\xi^i_k$, $1 \leq k \leq N_i$, of processes $\xi^i$, $i=1,2$, sampled at $m$ points $x_j$, $1 \leq j \leq m$
	\\
	\textbf{Procedure}:
	\begin{algorithmic}
		\STATE{Form $m \times N_i$ data matrices $Z_i$ , with $(Z_i)_{jk} = \xi^i_k(x_j)$, $i=1,2$, $1\leq j \leq m, 1\leq k \leq N_i$}
		\STATE{Compute $m \times m$ empirical covariance matrices $\hat{K}^i = \frac{1}{N}Z_iZ_i^{T}$, $i=1,2$ 
		}
		\STATE{Compute $W = W_2\left[\Ncal\left(0, \frac{1}{m}\hat{K}^1\right), \Ncal\left(0, \frac{1}{m}\hat{K}^2\right)\right]$ according to Eq.\eqref{equation:Gaussian-Wass-finite}}
		\STATE{Compute $S = \Srm^{\ep}_2\left[\Ncal\left(0, \frac{1}{m}\hat{K}^1\right), \Ncal\left(0, \frac{1}{m}\hat{K}^2\right)\right]$ according to Eq.\eqref{equation:gauss-sinkhorn-infinite}}
		\RETURN{$W$ and $S$}
	\end{algorithmic}
\end{algorithm}

\section{Estimation of Wasserstein distance between Gaussian processes}
\label{section:exact-Wasserstein}
In contrast to the Sinkhorn divergence of centered Gaussian processes, which is continuous in the $||\;||_{\HS}$ norm, the $2$-Wasserstein divergence is continuous in the $||\;||_{\tr}$ norm (\cite{masarotto2019procrustes,Bogachev:Gaussian}). 
We note that Theorem 8 in \cite{Mallasto:NIPS2017Wasserstein}, which claims that $W_2$ is continuous in the operator norm $||\;||$, is {\it not} correct in infinite-dimensional setting (see \cite{masarotto2019procrustes}, Proposition 4 and discussion).
More specifically (\cite{Minh:2021EntropicConvergenceGaussianMeasures}),
\begin{align}
W^2_2[\Ncal(0, C_1), \Ncal(0,C_2)] \leq ||C_1 - C_2||_{\tr}.
\end{align}
It is not clear if concentration results in e.g. \cite{Pinelis1994optimum}, which require $2$-smooth Banach space norms, can be extended to the $||\;||_{\tr}$ norm.
We now present estimates of
the
$2$-Wasserstein distance 
when
$\min\{\dim(\H_{K^i}), i=1,2 \}< \infty$, in which case $||\;||_{\tr(\H_{K^i})}$ and $||\;||_{\HS(\H_{K^i})}$ are equivalent for at least one $i=1,2$.
They are {\it not} valid in the case
$\dim(\H_{K^1}) = \dim(\H_{K^2}) = \infty$.

\begin{theorem}
[\textbf{Estimation of $2$-Wasserstein distance from finite covariance matrices}]
\label{theorem:Wasserstein-convergence-Gaussian-processes}
Under Assumptions $1-5$, let $\Xbf = (x_i)_{i=1}^m$ be independently sampled from $(T,\nu)$. Then
{\small
	\begin{align}
	W^2_2[\Ncal(0, C_{K^1}), \Ncal(0,C_{K^2})] &= \trace(L_{K^1}) + \trace(L_{K^2}) - 2 \trace[(R_{12}^{*}R_{12})^{1/2}].
	\label{equation:Wass-G-reprensetation}
\\
	W^2_2\left[\Ncal\left(0, \frac{1}{m}K^1[\Xbf]\right), \Ncal\left(0, \frac{1}{m}K^2[\Xbf]\right)\right]
	&= \trace(L_{K^1,\Xbf}) + \trace(L_{K^2, \Xbf}) - 2\trace[(R_{12,\Xbf}^{*}R_{12,\Xbf})^{1/2}].
	\label{equation:Wass-G-representation-empirical}
\end{align}
}
Assume further that $\dim(\H_{K^2}) < \infty$.
$\forall 0 < \delta < 1$, with probability at least $1-\delta$,
	\begin{align}
		&\left|W^2_2\left[\Ncal\left(0, \frac{1}{m}K^1[\Xbf]\right), \Ncal\left(0, \frac{1}{m}K^2[\Xbf]\right)\right]-	W^2_2[\Ncal(0, C_{K^1}), \Ncal(0,C_{K^2})]\right|
		\nonumber
		\\
		&\leq (\kappa_1^2 + \kappa_2^2)\left[ \frac{2\log\frac{6}{\delta}}{m} + \sqrt{\frac{2\log\frac{6}{\delta}}{m}}\right]
		+ 2\sqrt{2}\kappa_1\kappa_2\sqrt{\dim(\H_{K^2})}\sqrt{\frac{2\log\frac{6}{\delta}}{m} + \sqrt{\frac{2\log\frac{6}{\delta}}{m}}}. 
	\end{align}
\end{theorem}

\begin{theorem}
	\label{theorem:Wass-estimate-unknown-1}
	Assume Assumptions 1-5.
	Let $\Xbf = (x_i)_{i=1}^m \in T^m$, $m \in \Nbb$ be fixed.
	Let $\Wbf^1 = (\omega_j^1)_{j=1}^N$, $\Wbf^2 = (\omega_j^2)_{j=1}^N$ be independently sampled from
	$(\Omega_1, P_1)$ and $(\Omega_2, P_2)$, respectively.
	For any $0 < \delta<1$, with probability at least $1-\delta$,
	{\small
	\begin{align}
		&\left|W_2\left[\Ncal\left(0, \frac{1}{m}\hat{K}^1_{\Wbf^1}[\Xbf]\right), \Ncal\left(0, \frac{1}{m}\hat{K}^2_{\Wbf^2}[\Xbf]\right)\right]
		- W_2\left[\Ncal\left(0, \frac{1}{m}K^1[\Xbf]\right), \Ncal\left(0, \frac{1}{m}K^2[\Xbf]\right)\right]\right|
		\nonumber
		\\
		& \leq 3(\kappa_1 + \kappa_2)\sqrt{\frac{4}{N\delta^2} + \frac{\sqrt{m}}{\sqrt{N}\delta}\left(3 + \frac{4}{\sqrt{N}\delta}\right)}.
	\end{align}
}
\end{theorem}

\begin{theorem}
	[\textbf{Estimation of $2$-Wasserstein distance from finite samples}]
	\label{theorem:Wass-estimate-unknown-2}
		Assume Assumptions 1-5.
	Let $\Xbf = (x_i)_{i=1}^m$
	be independently sampled from $(T, \nu)$.
	Let $\Wbf^1 = (\omega_j^1)_{j=1}^N$, $\Wbf^2 = (\omega_j^2)_{j=1}^N$ be independently sampled from
	$(\Omega_1, P_1)$ and $(\Omega_2, P_2)$, respectively.
	For any $0 < \delta<1$, with probability at least $1-\delta$,
{\small
\begin{align}
&\left|W_2\left[\Ncal\left(0, \frac{1}{m}\hat{K}^1_{\Wbf^1}[\Xbf]\right), \Ncal\left(0, \frac{1}{m}\hat{K}^2_{\Wbf^2}[\Xbf]\right)\right] - W_2(\Ncal(0,C_{K^1}), \Ncal(0,C_{K^2}))\right|
\nonumber
\\
& \leq \left((\kappa_1^2 + \kappa_2^2)\left[ \frac{2\log\frac{12}{\delta}}{m} + \sqrt{\frac{2\log\frac{12}{\delta}}{m}}\right]
+ 2\sqrt{2}\kappa_1\kappa_2\sqrt{\dim(\H_{K^2})}\sqrt{\frac{2\log\frac{12}{\delta}}{m} + \sqrt{\frac{2\log\frac{12}{\delta}}{m}}}\right)^{1/2}
\\
&\quad +
3(\kappa_1 + \kappa_2)\sqrt{\frac{16}{N\delta^2} + \frac{2\sqrt{m}}{\sqrt{N}\delta}\left(3 + \frac{8}{\sqrt{N}\delta}\right)}.
\nonumber
\end{align}
}
\end{theorem}
In contrast to Theorems \ref{theorem:Sinkhorn-finite-sample-GaussianProcess} and \ref{theorem:Sinkhorn-estimate-unknown-2}, the convergence rates in Theorems \ref{theorem:Wasserstein-convergence-Gaussian-processes} and \ref{theorem:Wass-estimate-unknown-2}
both depend on $\dim(\H_{K^2})$. As an example, if $T = [0,1]^d$ and $K^2(x,y) = \la x,y\ra^D$, $D \in \Nbb$,
then \cite{CuckerSmale} $\dim(\H_{K^2}) = \begin{pmatrix} D +d-1\\d-1\end{pmatrix}$. 
Theorem \ref{theorem:Wass-estimate-unknown-2} provides the theoretical justification for the estimation of the $2$-Wasserstein distance in Algorithm \ref{algorithm:estimate-Sinkhorn} when  $\dim(\H_{K^2}) < \infty$.

\section{Estimation of Sinkhorn divergence via sample covariance operators}
\label{section:Sinkhorn-estimate-sample-covariance}
For comparison,
we now estimate
the Sinkhorn divergence via sample covariance operators, which is a standard approach 
in functional data analysis (see e.g. \cite{Panaretos:jasa2010,Horvath:2012functional}).
For $\xi \sim \GP(0,K)$ on the probability space $(\Omega, \Fcal, P)$, define the rank-one operator $\xi(\omega,.) \otimes \xi(\omega,.) \in \Lcal(\Lcal^2(T,\nu))$
by 	$[\xi(\omega,.) \otimes \xi(\omega,.)]f(x) = \xi(\omega,x)\int_{T}\xi(\omega,t)f(t)d\nu(t)$, 
$\omega \in \Omega$, $f \in \Lcal^2(T,\nu)$. Then
\begin{align}
	||\xi(\omega,.)\otimes \xi(\omega,.)||_{\HS(\Lcal^2(T,\nu))} &= \int_{T}\xi(\omega,t)^2d\nu(t) < \infty \;\;\text{$P$-almost surely}.
\end{align}
Thus $[\xi(\omega,.) \otimes \xi(\omega,.)] \in \HS(\Lcal^2(T,\nu))$ $P$-almost surely.
By Fubini Theorem (Lemma \ref{lemma:Fubini-switch}), 
\begin{align}
	C_K &= \bE[\xi \otimes \xi],\;\;
	C_Kf(x) = \bE\int_{T}\xi(\omega,x)\xi(\omega, t)f(t)d\nu(t)= \int_{T}K(x,t)f(t)d\nu(t).
\end{align}
Let $\Wbf = (\omega_j)_{j=1}^N$ be independently sampled from $(\Omega,P)$, corresponding to the samples
$\{\xi_j(t) = \xi(\omega_j,t)\}_{j=1}^N$ from $\xi$. It defines the pair of {\it sample covariance function/operator}
\begin{align}
	K_{\Wbf}(x,y) &= \frac{1}{N}\sum_{i=1}^N\xi(\omega_i,x)\xi(\omega_i,y),	\;\;C_{K,\Wbf} = \frac{1}{N}\sum_{i=1}^N\xi(\omega_i, .) \otimes \xi(\omega_i, .),
	\\
	C_{K,\Wbf}f(x) &= \frac{1}{N}\sum_{i=1}^N\int_{T}\xi(\omega_i, x)\xi(\omega_i,t)f(t)d\nu(t) = \int_{T}K_{\Wbf}(x,t)f(t)d\nu(t).
\end{align}
For each fixed $\Wbf$,
$K_{\Wbf}$ is 
symmetric, positive definite. It is continuous if the sample paths $\xi(\omega, .)$ are continuous $P$-almost surely, but not 
necessarily uniformly bounded over $\Wbf$ even if $K$ is bounded.
We always have, however, that $\H_{K_{\Wbf}} = \myspan\{\xi(\omega_j,.)\}_{j=1}^N\subset \Lcal^2(T,\nu)$ is a vector space of dimension at most $N$, $C_{K,\Wbf}$ is a finite-rank operator, together with the following
\begin{lemma}
	\label{lemma:sample-covariance-expectation}
	Under Assumptions 1-4 and 6, taking expectation with respect to $\Wbf$ gives
\begin{align}
\bE\int_{T}K_{\Wbf}(x,x)^2d\nu(x) & \leq 3\int_{T}K(x,x)^2d\nu(x)
\leq 3\kappa^4.
\end{align}
In particular $\bP\{\Wbf \in (\Omega,P)^N: \int_{T}K_{\Wbf}(x,x)^2d\nu(x) \leq \frac{3\kappa^4}{\delta}\} \geq 1-\delta$ for any $0 <\delta < 1$.
\end{lemma}
%
\begin{proposition}
	\label{proposition-concentration-sample-cov-operator}
	Assume Assumptions 1-4 and 6. Let $\Wbf = (\omega_j)_{j=1}^N$ be independently sampled from $(\Omega,P)$. 
	For any $0 < \delta < 1$, with probability at least $1-\delta$,
	\begin{align}
		||C_{K,\Wbf}||_{\HS(\Lcal^2(T,\nu))} \leq \frac{2\kappa^2}{\delta},\;\;\;||C_{K,\Wbf} - C_K||_{\HS(\Lcal^2(T,\nu))} \leq \frac{2\sqrt{3}\kappa^2}{\sqrt{N}\delta}.
	\end{align}
\end{proposition}

Combining Theorem \ref{theorem:Sinkhorn-convergence} and Proposition \ref{proposition-concentration-sample-cov-operator}, we obtain the following result.
\begin{theorem}
	[\textbf{Estimation of Sinkhorn divergence between Gaussian processes from sample covariance operators}]
	\label{theorem:Sinkhorn-convergence-sample-covariance-operator}
	Under Assumptions 1-4 and 6, let $\Wbf^i = (\omega^i_j)_{j=1}^N$, $i=1,2$, be independently sampled from 
	$(\Omega_i, P_i)$.
	$\forall 0 < \delta < 1$, with probability at least $1-\delta$,
	\begin{align}
		&\left|\Srm^{\ep}_2[\Ncal(0,C_{K^1,\Wbf^1}), \Ncal(0,C_{K^2,\Wbf^2})] - \Srm^{\ep}_2[\Ncal(0,C_{K^1}),\Ncal(0,C_{K^2})]\right|
		\nonumber
		\\
		&\leq \frac{12\sqrt{3}}{\ep\delta}\left(\left(1+\frac{4}{\delta}\right)\kappa_1^4 + \left(3+\frac{8}{\delta}\right)\kappa_1^2\kappa^2 + \kappa_2^4\right)\frac{1}{\sqrt{N}}.
	\end{align}
\end{theorem}
On a set $\Xbf = (x_i)_{i=1}^m \in T^m$, the Gram matrix of $K_{\Wbf}$ is precisely $\hat{K}_{\Wbf}[\Xbf]$,
as in Eq.\eqref{equation:K-hat-W}.
Combining Lemma \ref{lemma:sample-covariance-expectation}, Theorem \ref{theorem:Sinkhorn-finite-sample-GaussianProcess-unbounded}, and Theorem \ref{theorem:Sinkhorn-convergence-sample-covariance-operator}, we obtain the following result
\begin{theorem}
		[\textbf{Estimation of Sinkhorn divergence between Gaussian processes from finite samples - general kernels}]
	\label{theorem:Sinkhorn-convergence-sample-covariance-operator-finite-sample}
	Under Assumptions 1-4 and 6,
	let $\Wbf^1 = (\omega_j^1)_{j=1}^N$, $\Wbf^2 = (\omega_j^2)_{j=1}^N$ be independently sampled from
	$(\Omega_1, P_1)$ and $(\Omega_2, P_2)$, respectively. Let $\Xbf = (x_i)_{i=1}^m$
	be independently sampled from $(T, \nu)$.
	For any $0 < \delta<1$, with probability at least $1-\delta$,
{\small
	\begin{align}
		&\left|\Srm^{\ep}_2\left[\Ncal\left(0, \frac{1}{m}\hat{K}^1_{\Wbf^1}[\Xbf]\right), \Ncal\left(0, \frac{1}{m}\hat{K}^2_{\Wbf^2}[\Xbf]\right)\right] - \Srm^{\ep}_2[\Ncal(0,C_{K^1}), \Ncal(0,C_{K^2})]\right|
		\nonumber
		\\
		&\leq \frac{48\sqrt{3}}{\ep\delta}\left(\left(1+\frac{16}{\delta}\right)\kappa_1^4 + \left(3+\frac{32}{\delta}\right)\kappa_1^2\kappa^2 + \kappa_2^4\right)\frac{1}{\sqrt{N}}
+\frac{864}{\ep\delta}(\kappa_1^2 + \kappa_2^2)^2\left(1+\frac{12}{\delta}\right)\frac{1}{\sqrt{m}}.
		\end{align}
	}
\end{theorem}
We note that a similar version of Theorem \ref{theorem:Sinkhorn-convergence-sample-covariance-operator-finite-sample}
can also be obtained from Theorem \ref{theorem:Sinkhorn-finite-sample-GaussianProcess-unbounded}.

\section{Divergences between covariance operators of stochastic processes}
\label{section:general-stochastic-process}
Assume
 that $\xi^1, \xi^2$ are centered stochastic processes, {\it not} necessarily Gaussian, with covariance functions $K^1,K^2$ and
paths in $\Lcal^2(T,\nu)$.  Then
$W_2[\Ncal(0, C_{K^1}), \Ncal(0, C_{K^2})]$,
$\Srm^{\ep}_2[\Ncal(0, C_{K^1}), \Ncal(0, C_{K^2})]$ are distance/divergence
between the two covariance operators $C_{K^1}, C_{K^2}$ associated with $\xi^1,\xi^2$.

{\bf Assumption 6(*)}
$\xi, \xi^i$, $i=1,2$, are centered stochastic processes
and $\exists \kappa, \kappa_i > 0$ with
\begin{align}
	\label{equation:assumption-general}
	\bE\int_{T}\xi(\omega, x)^4d\nu(x) \leq 3 \kappa^4,\;\;
\bE\int_{T}\xi^i(\omega, x)^4d\nu(x) \leq 3\kappa_i^4.
\end{align}
For $\xi \sim \Ncal(0, K)$, $\xi^i \sim \Ncal(0, K^i)$, Assumption 6(*) reduces to Assumption 6, as follows.
\begin{lemma}
	\label{lemma:fourth-moment-Gaussian-process}
	Under Assumption 6(*), 
	$\xi(\omega, .) \in \Lcal^2(T,\nu)$ $P$-almost surely. Furthermore, 
	$\int_{T}K(x,x)^2d\nu(x) \leq 3\kappa^4$ and $\bE||\xi||^4_{\Lcal^2(T,\nu)} \leq 3\kappa^4$.
	In particular, if $\xi \sim \Ncal(0, K)$, then
	\begin{align}
		\bE[\xi(.,x)^4] = 3K(x,x)^2 \;\forall x \in T, \;\;\bE\int_{T}\xi(\omega, x)^4d\nu(x) = 3\int_{T}K(x,x)^2d\nu(x).
	\end{align}
\end{lemma}
Hence, Proposition \ref{proposition:concentration-TK2K1-empirical-unbounded}, 
\ref{proposition-concentration-sample-cov-operator}, and Theorems
\ref{theorem:Sinkhorn-finite-sample-GaussianProcess-unbounded} \ref{theorem:Sinkhorn-convergence-sample-covariance-operator}, \ref{theorem:Sinkhorn-convergence-sample-covariance-operator-finite-sample} for $\Srm^{\ep}_2$ carry over virtually unchanged, except
for some absolute constant factors. We note that
condition 
$\bE||\xi||^4_{\Lcal^2(T,\nu)} \leq 3\kappa^4$ is sufficient for proving Proposition \ref{proposition-concentration-sample-cov-operator} and Theorem \ref{theorem:Sinkhorn-convergence-sample-covariance-operator}. Similar results also hold
for $W_2$.

\section{Numerical experiments}
\label{section:experiments}

We demonstrate $W_2$ and $\S^{\ep}_2$ 
on the following Gaussian processes
$\xi^i = \GP(0,K^i)$, $i=1,2$, on $T = [0,1]^d \subset \R^d$, where $d = 1,5,50$, with
\begin{align}
	\label{equation:GP-process-example}
K^1(x,y) = \exp(-a||x-y||),\;\;\; K^2(x,y) = \exp\left(-\frac{1}{\sigma^2}||x-y||^2\right),
\end{align}
In the experiments, we fix $a = 1,\sigma = 0.1$.
Figure \ref{figure:Gaussian-process-sample} shows samples of these processes for $d=1$.

(i) Let $\Xbf = (x_i)_{i=1}^m$ be randomly chosen from $T$, where $m=10,20,30,\ldots, 1000$. 
We plot in Figures \ref{figure:Gaussian-process-sample} and \ref{figure:different-divergences} the following divergences between $(1/m)K^1[\Xbf]$ and $(1/m)K^2[\Xbf]$: $||\;||^2_{\HS}$ (squared Hilbert-Schmidt), 
$W^2_2$ (squared Wasserstein), and $\Srm^{\ep}_2$ (Sinkhorn, $\ep =0.1$ and $\ep=0.5$). 

(ii) Consider two sets of $N$ sample paths from each process, each path sampled at $m=500$ points
$\Xbf =(x_i)_{i=1}^m$, which are randomly chosen and fixed in advance from $T$.
We then  compute the different divergences using Algorithm \ref{algorithm:estimate-Sinkhorn}, for $N =10, 20, \ldots, 1000$ (Figure \ref{figure:different-divergences-approx}).

In agreement with theory, 
the convergence of the Sinkhorn divergence and Hilbert-Schmidt distance, being dimension-independent, is consistent across different dimensions,
whereas the convergence of the Wasserstein distance is slower the larger the dimension $d$ is.

\begin{figure}
	\begin{subfigure}{0.3\textwidth}
		\includegraphics[width=\textwidth]{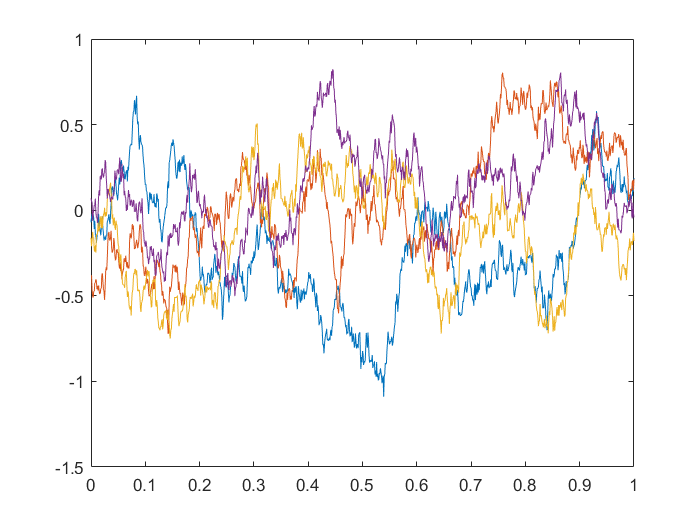}
		\end{subfigure}
	\begin{subfigure}{0.3\textwidth}
		\includegraphics[width=\textwidth]{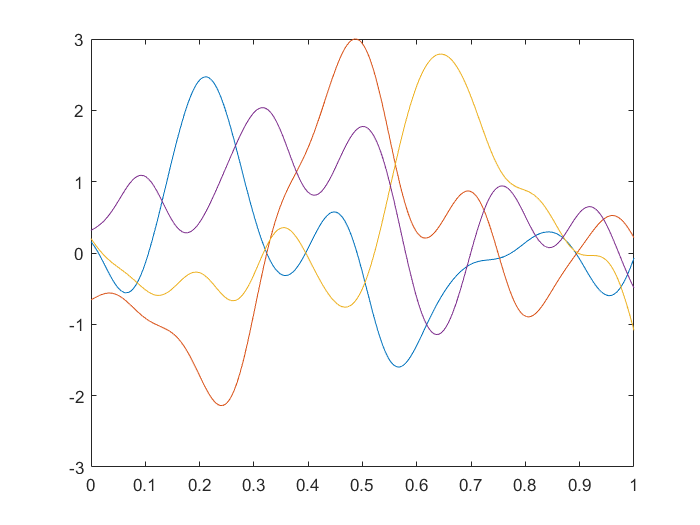}
		\end{subfigure}
	\begin{subfigure}{0.3\textwidth}
		\includegraphics[width =\textwidth]{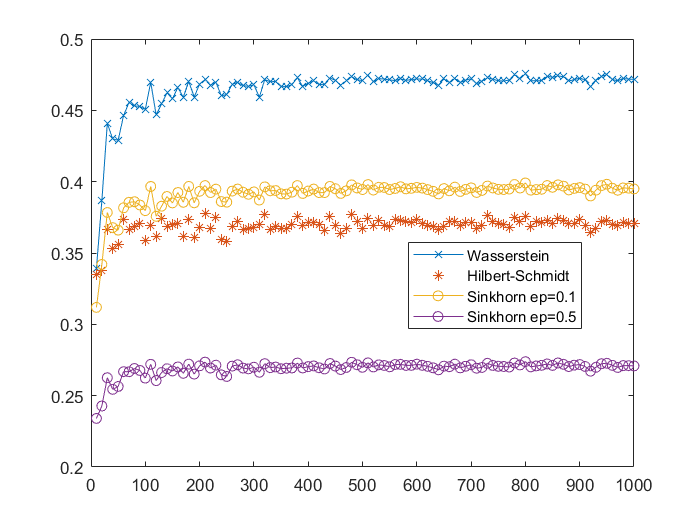}
		\end{subfigure}
		\caption{Samples of the centered Gaussian processes defined in Eq.\eqref{equation:GP-process-example}
		on $T = [0,1]$ and approximations of divergences/squared distances between them.
		Left: $K^1(x,y) = \exp(-a||x-y||)$, $a=1$. Right: $K^2(x,y) = \exp(-||x-y||^2/\sigma^2)$, $\sigma = 0.1$.
	Here $m =10,20,\ldots, 1000$.}.
		\label{figure:Gaussian-process-sample}
\end{figure}

\begin{figure}
	\begin{center}
		\includegraphics[width = 0.35\textwidth]{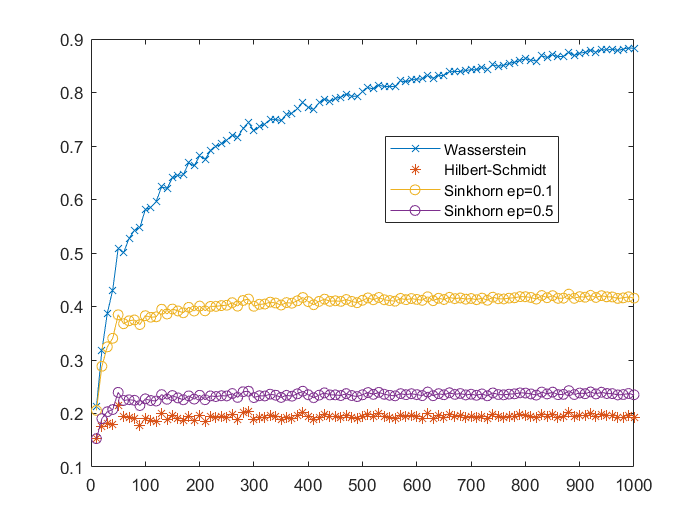}
		\includegraphics[width = 0.35\textwidth]{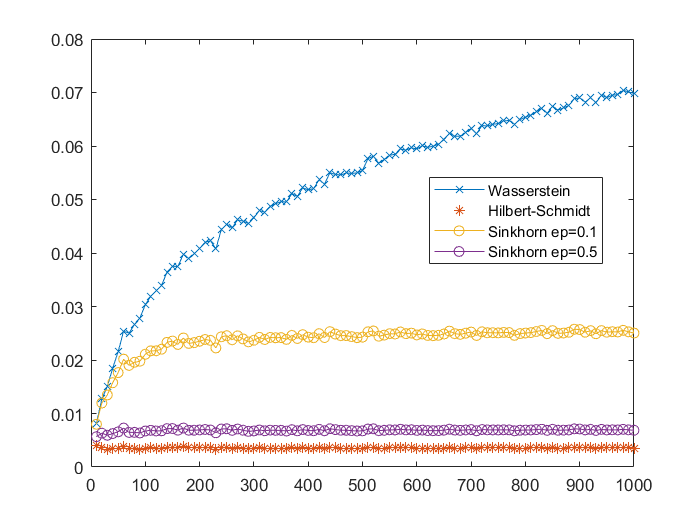}
		\caption{Approximate divergences/squared distances between  the Gaussian processes defined in Eq.\eqref{equation:GP-process-example}, using normalized $m \times m$ covariance matrices,
			on $T = [0,1]^d \subset \R^d$. Left: $d = 5$. Right: $d=50$.
	Here $m =10,20,\ldots, 1000$.	
	}
		\label{figure:different-divergences}
	\end{center}
\end{figure}

\begin{figure}
	\begin{center}
		\includegraphics[width = 0.3\textwidth]{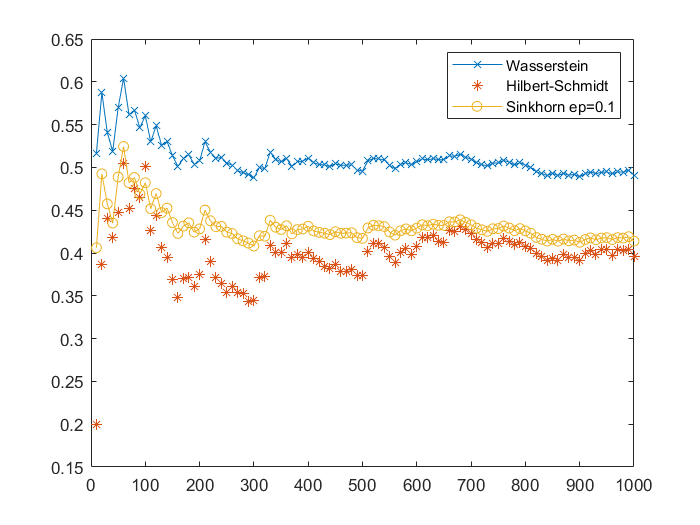}
		\includegraphics[width = 0.3\textwidth]{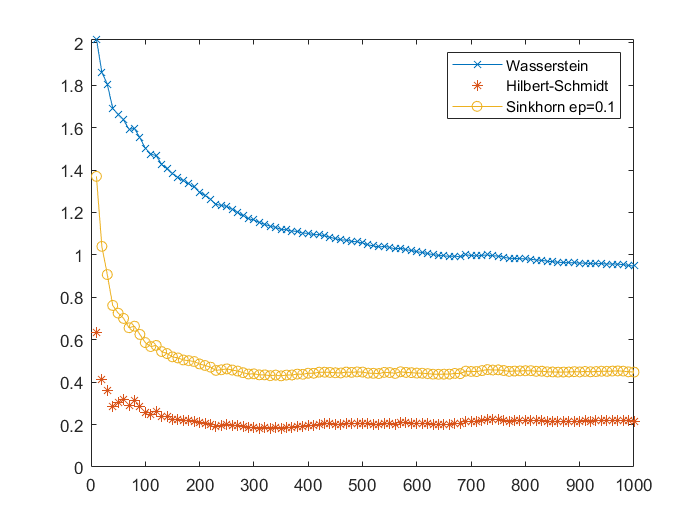}
		\includegraphics[width = 0.3\textwidth]{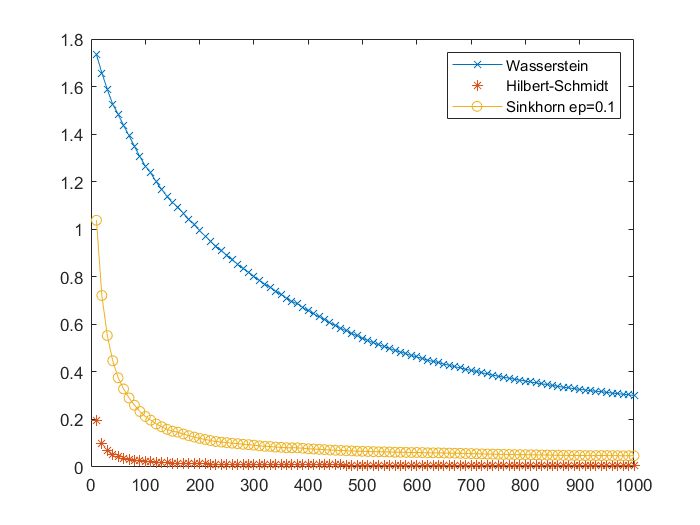}
		\caption{Approximate divergences/squared distances between the Gaussian processes defined in Eq.\eqref{equation:GP-process-example}
			on $T = [0,1]^d \subset \R^d$. Left: $d = 1$. Middle: $d=5$. Right: $d=50$.
			The estimation is obtained using $N$ realizations of each process, sampled at $m=500$ points,
			according to Algorithm \ref{algorithm:estimate-Sinkhorn}.
			Here $N =10,20,\ldots, 1000$.	
		}
		\label{figure:different-divergences-approx}
	\end{center}
\end{figure}

\section{Proofs of main results}
\label{section:proofs}

\begin{proof}
	[\textbf{Proof of Lemma \ref{lemma:R12-HS}}]
	It suffices to prove this for the case $i=1,j=2$.
	Let $\{e_k\}_{k\in \Nbb}$ be any orthonormal basis in $\H_{K^2}$, then
	\begin{align*}
		&||R_{12}||_{\HS(\H_{K^2}, \H_{K^1})}^2 = \sum_{k=1}^{\infty}||R_{12}e_k||^2_{\H_{K^1}} = \sum_{k=1}^{\infty}\left\|\int_{T}K^1_x \la e_k, K^2_x\ra_{\H_{K^2}}d\nu(x)\right\|^2_{\H_{K^1}}
		\\
		& \leq \sum_{k=1}^{\infty}\left(\int_{T}||K^1_x||_{\H_{K^1}}|\la e_k, K^2_x\ra_{\H_{K^2}}|d\nu(x)\right)^2
		\leq \sum_{k=1}^{\infty}\int_{T}||K^1_x||^2_{\H_{K^1}}d\nu(x)
		\int_{T}|\la e_k, K^2_x\ra_{\H_{K^2}}|^2d\nu(x)
		\\
		& =\int_{T}||K^1_x||^2_{\H_{K^1}}d\nu(x)\int_{T}\sum_{k=1}^{\infty}|\la e_k, K^2_x\ra_{\H_{K^2}}|^2d\nu(x)
		\;\text{by the Monotone Convergence Theorem}
		\\
		& = \int_{T}||K^1_x||^2_{\H_{K^1}}d\nu(x)\int_{T}||K^2_x||^2_{\H_{K^2}}d\nu(x)
		=\int_{T}K^1(x,x)d\nu(x)\int_{T}K^2(x,x)d\nu(x) \leq \kappa_1^2\kappa_2^2.
	\end{align*}
\end{proof}

\begin{proof}
	[\textbf{Proof of Lemma \ref{lemma:R12-empirical}}]
	For any $f \in \H_{K^2}$, $g \in \H_{K^1}$,
	\begin{align*}
		&\la g, R_{12,\Xbf}f\ra_{\H_{K^1}} = \left\la g, \frac{1}{m}\sum_{i=1}^mf(x_i)K^1_{x_i}\right\ra_{\H_{K^1}} = \frac{1}{m}\sum_{i=1}^mf(x_i)g(x_i)=\la R_{21,\Xbf}g, f\ra_{\H_{K^2}},
	\end{align*}
	showing that $R_{12,\Xbf}^{*} = R_{21,\Xbf}$. It follows that for any $f \in \H_{K^2}$,
	\begin{align*}
		&R_{12,\Xbf}^{*}R_{12,\Xbf}f = \frac{1}{m}\sum_{i=1}^mR_{12,\Xbf}^{*}f(x_i)K^1_{x_i} = \frac{1}{m^2}\sum_{i,j=1}^mf(x_i)K^2_{x_j}\la K^1_{x_i}, K^1_{x_j}\ra_{\H_{K^1}}
		\\
		& = \frac{1}{m^2}\sum_{i,j=1}^mf(x_i)K^1(x_i, x_j)K^2_{x_j}.
\\
		&R_{12,\Xbf}^{*}R_{12,\Xbf}K^2_{x_k} = \frac{1}{m^2}\sum_{i,j=1}^mK^2(x_k,x_i)K^1(x_i,x_j)K^2_{x_j}
		= \frac{1}{m^2}\sum_{j=1}^m(K^2[\Xbf]K^1[\Xbf])_{kj}K^2_{x_j}.
	\end{align*}
	It follows that, in the $\myspan\{K^2_{x_i}\}_{i=1}^m$, the matrix representation of $R_{12,\Xbf}^{*}R_{12,\Xbf}:\H_{K^2,\Xbf} \mapto \H_{K^2,\Xbf}$
	is $\frac{1}{m^2}(K^2[\Xbf]K^1[\Xbf])^T = \frac{1}{m^2}K^1[\Xbf]K^2[\Xbf]$.
\end{proof}

\begin{lemma}
	[\textbf{Corollary 5 in \cite{Minh:2021EntropicConvergenceGaussianMeasures}}]
	\label{lemma:square-root-trace-norm}
	For $A,B \in \Sym^{+}(\H) \cap \Tr(\H)$,
	\begin{align}
		||(I+A)^r - (I+B)^r||_{\tr} \leq r ||A-B||_{\tr},  \;\; 0 \leq r \leq 1.
	\end{align}
\end{lemma}
\begin{corollary}
	\label{corollary:trace-square-root-trace-norm}
	For $A,B \in \Sym^{+}(\H) \cap \Tr(\H)$,
	\begin{align}
		|\trace[-I+(I+A)^{1/2}] -\trace[-I+ (I+B)^{1/2}]| \leq \frac{1}{2}||A-B||_{\tr}.
	\end{align}
\end{corollary}

\begin{lemma}
	[\textbf{Corollary 6 in \cite{Minh:2021EntropicConvergenceGaussianMeasures}}]
	\label{lemma:logdet-square-root-trace-norm}
	For $A,B \in \Sym^{+}(\H) \cap \Tr(\H)$,
	\begin{align}
		\left|\log\det\left(\frac{1}{2}I + \frac{1}{2}(I+A)^{1/2}\right)- \log\det\left(\frac{1}{2}I + \frac{1}{2}(I+B)^{1/2}\right)\right| \leq \frac{1}{4}||A-B||_{\tr}.
	\end{align}
\end{lemma}
\begin{proof}
	[\textbf{Proof of Proposition \ref{proposition:continuity-trace-norm-G}}]
	(i) 
	For $A,B\in \Sym^{+}(\H)\cap \Tr(\H)$,
	\begin{align*}
		&|G(A) - G(B)| \leq |\trace[M(A)] - \trace[M(B)]|
		+\left|\log\det\left(I + \frac{1}{2}M(A)\right) - \log\det\left(I+\frac{1}{2}M(B)\right)\right|
		\\
		& = |\trace[-I +(I+c^2A)^{1/2}]-\trace[-I + (I+c^2B)^{1/2}]|
		\\
		&+  \left|\log\det\left(\frac{1}{2}I + \frac{1}{2}(I+c^2A)^{1/2}\right) - \log\det\left(\frac{1}{2}I+\frac{1}{2}(I+c^2B)^{1/2}\right)\right|
		\\
		& \leq \frac{c^2}{2}||A-B||_{\tr} + \frac{c^2}{4}||A-B||_{\tr} = \frac{3c^2}{4}||A-B||_{\tr}, \text{by Corollary \ref{corollary:trace-square-root-trace-norm} and Lemma \ref{lemma:logdet-square-root-trace-norm}}.
	\end{align*}
	(ii) For $A,B \in \Sym(\H) \cap \HS(\H)$, using the first part,
	\begin{align*}
		|G(A^2) - G(B^2)| &\leq \frac{3c^2}{4}||A^2 - B^2||_{\tr}\leq \frac{3c^2}{4}||A(A-B) +(A-B)B||_{\tr}
		\\
		&\leq \frac{3c^2}{4}[||A||_{\HS} + ||B||_{\HS}]||A-B||_{\HS}.
	\end{align*}
	(iii) For $A,B \in \HS(\H_1, \H_2)$, by Proposition \ref{proposition:trace-norm-product-HS-different-spaces},
	$A^{*}A, B^{*}B \in \Sym^{+}(\H_1) \cap \Tr(\H_1)$ and
	\begin{align*}
		|G(A^{*}A) - G(B^{*}B)| &\leq \frac{3c^2}{4}||A^{*}A - B^{*}B||_{\tr(\H_1)} \;\;\text{from part (i)}
		\\
		&\leq
		\frac{3c^2}{4}[||A||_{\HS(\H_1, \H_2)} + ||B||_{\HS(\H_1, \H_2)}]||A-B||_{\HS(\H_1,\H_2)}.
	\end{align*}
\end{proof}

\begin{proof}
[\textbf{Proof of Proposition \ref{proposition:concentration-TK2K1-empirical}}]
It suffices to prove this for the case $i=1,j=2$.
Define the random variable $\xi: (T, \nu) \mapto \HS(\H_{K^2}, \H_{K^1})$
by $\xi(x) = K^1_x \otimes K^2_x:\H_{K^2} \mapto \H_{K^1}$, $\xi(x)f = K^1_x\la f, K^2_x\ra = f(x)K^1_x$
$\forall f \in \H_{K^2}$. Then
$\bE[\xi(x)] = \int_{T}(K^1_x \otimes K^2_x)d\nu(x) = R_{12}$,\;\;
$\frac{1}{m}\sum_{i=1}^m\xi(x_i) = \frac{1}{m}\sum_{i=1}^mK^1_{x_i} \otimes K^2_{x_i} = R_{12,\Xbf}$.
By Lemma \ref{lemma:HS-norm-rank-one-product}, 
\begin{align*}
&||\xi(x)||_{\HS(\H_{K^2},\H_{K^1})} = ||K^1_x||_{\H_{K^1}}||K^2_x||_{\H_{K^2}} = \sqrt{K^1(x,x)K^2(x,x)} \leq \kappa_1\kappa_2, \;\;\forall x \in T,
\\
&||R_{12,\Xbf}||_{\HS(\H_{K^2}, \H_{K^1})} \leq \frac{1}{m}\sum_{i=1}^m||K^1_{x_i} \otimes K^2_{x_i}||_{\HS(\H_{K^2}, \H_{K^1})} \leq \kappa_1\kappa_2,
\\
&E||\xi||^2_{\HS(\H_{K^2},\H_{K^1})} = \int_{T}K^1(x,x)K^2(x,x)d\nu(x) \leq \kappa_1^2\kappa_2^2.
\end{align*}
The
bound for $||R_{12,\Xbf} - R_{12}||_{\HS(\H_{K^2}, \H_{K^1})}$
follows from
 Proposition \ref{proposition:Pinelis}. By Corollary \ref{corollary:trace-norm-AB-different-spaces},
\begin{align*}
&||R_{12,\Xbf}^{*}R_{12,\Xbf} - R_{12}^{*}R_{12}||_{\tr(\H_{K^2})} 
\\
&\leq (||R_{12,\Xbf}||_{\HS(\H_{K^2},\H_{K^1})} + ||R_{12}||_{\HS(\H_K^2,\H_{K^1})})||R_{12,\Xbf}-R_{12}||_{\HS(\H_{K^2}, \H_{K^1})},
\end{align*}
which gives the desired bound for $||R_{12,\Xbf}^{*}R_{12,\Xbf} - R_{12}^{*}R_{12}||_{\tr(\H_{K^2})}$.
\end{proof}

\begin{proof}
	[\textbf{Proof of Theorem \ref{theorem:Sinkhorn-finite-sample-GaussianProcess}}]
	By Proposition \ref{propopsition:Sinkhorn-cross-covariance}, 
	let $c = \frac{4}{\ep}$, then
	\begin{align*}
		&\Delta = \left|\Srm^{\ep}_2[\Ncal(0,C_{K^1}), \Ncal(0, C_{K^2})] - \Srm^{\ep}_2\left[\Ncal\left(0, \frac{1}{m}K^1[\Xbf]\right), \Ncal\left(0, \frac{1}{m}K^2[\Xbf]\right)\right] \right|
		\\
		& \leq \frac{1}{c}\left[|G(L_{K^1,\Xbf}^2) - G(L_{K^1}^2)| + |G(L_{K^2,\Xbf}^2) - G(L_{K^2}^2)|
		+2|G(R_{12,\Xbf}^{*}R_{12,\Xbf}) - G(R_{12}^{*}R_{12})|\right].
	\end{align*}
	We now combine Proposition \ref{proposition:continuity-trace-norm-G} with Propositions  
	\ref{proposition:concentration-TK2K1-empirical}. For $0 < \delta < 1$, define
	{\small
	\begin{align*}
	U_i &= \left\{\Xbf \in T^m: |G(L_{K^i,\Xbf}^2) - G(L_{K^i}^2)| \leq \frac{3c^2}{2} \kappa_i^4\left(\frac{2\log\frac{6}{\delta}}{m} + \sqrt{\frac{2\log\frac{6}{\delta}}{m}}\right)\right\},\;\;i=1,2,
\\
U_3 &= \left\{\Xbf \in T^m: |G(R_{12,\Xbf}^{*}R_{12,\Xbf}) - G(R_{12}^{*}R_{12})| \leq \frac{3c^2}{2}\kappa_1^2\kappa_2^2\left[ \frac{2\log\frac{6}{\delta}}{m} + \sqrt{\frac{2\log\frac{6}{\delta}}{m}}\right]\right\}.
\end{align*}
}
Then $\nu^m(U_1) \geq 1-\frac{\delta}{3}$, $\nu^m(U_2) \geq 1-\frac{\delta}{3}$, $\nu^m(U_3)\geq 1-\frac{\delta}{3}$.
From the Inclusion-Exclusion Principle, using the property $\nu^m(U_i \cup U_j) = \nu^m(U_i) + \nu^m(U_j) - \nu^m(U_i \cap U_j)$,
\begin{align*}
\nu^m(U_1 \cap U_2 \cap U_3) &= \nu^m(U_1\cup U_2 \cup U_3) + \nu^m(U_1) + \nu^m(U_2)  + \nu^m(U_3)
\\
&  - [\nu^m(U_1 \cup U_2) + \nu^m(U_1\cup U_3) + \nu^m(U_2 \cup U_3)]
\geq 1 + 3(1-\frac{\delta}{3}) - 3 = 1-\delta.
\end{align*}
Thus for any $0 < \delta < 1$, with probability at least $1-\delta$,
\begin{align*}
\Delta \leq \frac{3c}{2}(\kappa_1^2 + \kappa_2^2)^2\left[ \frac{2\log\frac{6}{\delta}}{m} + \sqrt{\frac{2\log\frac{6}{\delta}}{m}}\right] = \frac{6}{\ep}(\kappa_1^2 + \kappa_2^2)^2\left[ \frac{2\log\frac{6}{\delta}}{m} + \sqrt{\frac{2\log\frac{6}{\delta}}{m}}\right].
\end{align*}
\end{proof}

\begin{proof}
	[\textbf{Proof of Propositions \ref{proposition:concentration-TK2K1-empirical-unbounded}}]	
	Define the random variables $Y_j: (T, \nu)^m \mapto \HS(\H_{K^2}, \H_{K^1})$ by $Y_j(\Xbf) = K^1_{x_j}\otimes K^2_{x_j}$, $1\leq j \leq m$,
	where $\Xbf = (x_j)_{j=1}^m$ is independently sampled from $(T,\nu)$. The $Y_j$'s are IID, with
	$\bE{Y_j}
	 = R_{12}$ and
	$\frac{1}{m}\sum_{j=1}^mY_j(\Xbf) = R_{12,\Xbf}$. By Lemma \ref{lemma:HS-norm-rank-one-product},
	\begin{align*}
		\bE||Y_j||_{\HS(\H_{K^2},\H_{K^1})}& = \int_{T}||K^1_x||_{\H_{K^1}}||K^2_x||_{\H_{K^2}}d\nu(x) = \int_{T}\sqrt{K^1(x,x)}\sqrt{K^2(x,x)}d\nu(x) 
		\\
		& \leq \sqrt{\int_{T}K^1(x,x)d\nu(x)}\sqrt{\int_{T}K^2(x,x)d\nu(x)} \leq \kappa_1\kappa_2,
		\\
		\bE||Y_j||^2_{\HS(\H_{K^2},\H_{K^1})}& = \int_{T}||K^1_x||^2_{\H_{K^1}}||K^2_x||^2_{\H_{K^2}} =
		\int_{T}K^1(x,x)K^2(x,x)d\nu(x) 
		\\
		& \leq \sqrt{\int_{T}[K^1(x,x)]^2d\nu(x)}\sqrt{\int_{T}[K^2(x,x)]^2d\nu(x)} \leq \kappa_1^2\kappa_2^2. 
	\end{align*}
	Define the random variable $\eta: (T, \nu)^m \mapto \R$ by
	$\eta(\Xbf) =\left\|\frac{1}{m}\sum_{j=1}^mY_j(\Xbf) - \bE{Y_j}\right\|_{\HS(\H_{K^2}, \H_{K^1})} 
	= ||R_{12,\Xbf} - R_{12}||_{\HS(\H_{K^2}, \H_{K^1})}$. Since the $Y_j$'s are IID,
	{\small
		\begin{align*}
			\bE{\eta^2} &= \bE\left\|\frac{1}{m}\sum_{j=1}^mY_j - \bE{Y_j}\right\|_{\HS}^2 =\frac{1}{m^2}\sum_{j=1}^m\bE||Y_j - EY_j||^2_{\HS}
			= \frac{1}{m^2}\sum_{j=1}^m(\bE||Y_j||^2_{\HS} - ||EY_j||^2_{\HS}) \leq \frac{\kappa_1^2\kappa_2^2}{m}.
		\end{align*}
	}
	By Chebyshev inequality, for any $t > 0$,
	$\bP(\eta \geq t) \leq \frac{\bE{\eta}}{t} \leq \frac{\sqrt{\bE{\eta^2}}}{t} \leq \frac{\kappa_1\kappa_2}{\sqrt{m}t}$.
	Let $\delta = \frac{\kappa_1\kappa_2}{\sqrt{m}t} \equivalent t = \frac{\kappa_1\kappa_2}{\sqrt{m}\delta}$,
	then 
	$\bP\{\Xbf: ||R_{12,\Xbf} - R_{12}||_{\HS} = \eta(\Xbf) \leq \frac{\kappa_1\kappa_2}{\sqrt{m}\delta}\} \geq 1 -\delta$.
	Similarly, since $\bE{||R_{12,\Xbf}||_{\HS}} \leq \frac{1}{m}\sum_{j=1}^m\bE{||Y_j||_{\HS}} \leq \kappa_1\kappa_2$, we have
	$\bP\{\Xbf: ||R_{12,\Xbf}||_{\HS} \leq \frac{\kappa_1\kappa_2}{\delta}\} \geq 1-\delta$.
	Computing the intersection of these two sets of events and replacing $\delta$ by $\frac{\delta}{2}$, we obtain the desired bounds.
\end{proof}

\begin{proof}
	[\textbf{Proof of Theorem \ref{theorem:Sinkhorn-finite-sample-GaussianProcess-unbounded}}]
	Define $\Delta$ as in the proof of Theorem \ref{theorem:Sinkhorn-finite-sample-GaussianProcess}. 
	We now combine Proposition \ref{proposition:continuity-trace-norm-G} with Propositions  
	\ref{proposition:concentration-TK2K1-empirical-unbounded}. For $0 < \delta < 1$, define
	\begin{align*}
			U_i &= \left\{\Xbf \in T^m: |G(L_{K^i,\Xbf}^2) - G(L_{K^i}^2)| \leq \frac{3c^2}{2}\kappa_i^4\left(1+\frac{6}{\delta}\right)\frac{3}{\sqrt{m}\delta}\right\}, \;\;i=1,2,
	\\
		U_3 &= \left\{\Xbf \in T^m: |G(R_{12,\Xbf}^{*}R_{12,\Xbf}) - G(R_{12}^{*}R_{12})| \leq
		\frac{3c^2}{2}\kappa_1^2\kappa_2^2\left(1+\frac{6}{\delta}\right)\frac{3}{\sqrt{m}\delta}\right\}.
	\end{align*}
Then $\nu^m(U_i) \geq 1-\frac{\delta}{3}$, $i=1,2,3$.
	Thus for any $0 < \delta < 1$, with probability at least $1-\delta$,
	\begin{align*}
		\Delta \leq \frac{3c}{2}(\kappa_1^2+\kappa_2^2)^2\left(1+\frac{6}{\delta}\right)\frac{3}{\sqrt{m}\delta}
		= \frac{18}{\ep}(\kappa_1^2+\kappa_2^2)^2\left(1+\frac{6}{\delta}\right)\frac{1}{\sqrt{m}\delta}.
	\end{align*}
\end{proof}

\begin{proof}
	[\textbf{Proof of Lemma \ref{lemma:fourth-moment-Gaussian-process}}]
	Since
	$\bE||\xi||^4_{\Lcal^2(T,\nu)} = \bE\left(\int_{T}\xi(\omega,x)^2d\nu(x)\right)^2
	\leq \bE\int_{T}\xi(\omega,x)^4d\nu(x)$ $\leq 3\kappa^4$, we have
	$\bE||\xi||^2_{\Lcal^2(T,\nu)}\leq \sqrt{\bE||\xi||^4_{\Lcal^2(T,\nu)}} \leq \sqrt{3}\kappa^2$, thus
	$\xi(\omega, .)\in \Lcal^2(T,\nu)$ $P$-almost surely.
	By H\"older Theorem and Tonelli Theorem,
	\begin{align*}
		\int_{T}K(x,x)^2d\nu(x) &= \int_{T}\left(\int_{\Omega}\xi(\omega, x)^2dP(\omega)\right)^2d\nu(x)
		\leq \int_{\Omega}\int_{T}\xi(\omega, x)^4dP(\omega)d\nu(x) \leq 3\kappa^4.
	\end{align*}
	If $\xi \sim \Ncal(0,K)$, then for each fixed $x \in T$, we have $\xi(.,x) \sim \Ncal(0, K(x,x))$. Thus
	\begin{align*}
		\bE{\xi(.,x)^4} = \int_{\Omega}\xi(\omega, x)^4  dP(\omega) = \int_{\R}t^4 d\Ncal(0,K(x,x))(t) = 3K(x,x)^2
	\end{align*}
	by using the integral $\int_{\R}t^4d\Ncal(0,\lambda)(t) = 3\lambda^2$ (see Formula 7.4.4 in \cite{Handbook:1972}).
\end{proof}

\begin{proof}
	[\textbf{Proof of Proposition \ref{proposition:concentration-empirical-covariance}}]
	We have $||K[\Xbf]||_F = m||L_{K,\Xbf}||_{\HS(\H_K)} \leq m\kappa^2$ by Proposition 
	\ref{proposition:concentration-TK2K1-empirical}.
	Define the map 
	$\zbf:\Omega \mapto \R^m$ by  $\zbf(\omega) 
	=(\xi(\omega, x_i))_{i=1}^m \in \R^m$.
	Let $Y_j:(\Omega, P)^N \mapto \Sym^{+}(m)$, $1\leq j \leq N$, be IID 
	$\Sym^{+}(m)$-valued 
	random variables defined by
	$Y_j(\Wbf) = \zbf(\omega_j)\zbf(\omega_j)^T$, where
	$\Wbf = (\omega_1, \ldots,\omega_N)$ is independently sampled from $(\Omega,P)^N$. Then
	\begin{align*}
		K[\Xbf] &= \int_{\Omega}\zbf(\omega)\zbf(\omega)^TdP(\omega) = \bE{Y_j},\;\;
		\hat{K}_{\Wbf}[\Xbf] = \frac{1}{N}\sum_{j=1}^N\zbf(\omega_j)\zbf(\omega_j)^T = \frac{1}{N}\sum_{j=1}^NY_j(\Wbf),
		\\
		\bE||Y_j||_F &= \bE||\zbf(\omega)||^2 = \int_{\Omega}\sum_{i=1}^m\xi(\omega,x_i)^2dP(\omega) 
		= \sum_{i=1}^mK(x_i,x_i) \leq m\kappa^2,
		\\
		\bE||Y_j||^2_F & = \bE||\zbf||^4 = \bE\left[(\sum_{i=1}^m\xi(\omega,x_i)^2)^2\right]
		\leq m \sum_{i=1}^m\bE\left[\xi(\omega, x_i)^4\right] = 3m\sum_{i=1}^mK(x_i,x_i)^2 \leq 3m^2\kappa^4
	\end{align*}
by Lemma \ref{lemma:fourth-moment-Gaussian-process}.
	Define
	$\eta:(\Omega, P)^N \mapto \R$ by
	$	\eta(\Wbf) = \left\|\frac{1}{N}\sum_{j=1}^NY_j(\Wbf) - \bE{Y_j}\right\|_F$ $= ||\hat{K}_{\Wbf}[\Xbf] - K[\Xbf]||_F$.
	Since the $Y_j$'s are 
	independent, identically distributed, 
	{\small
	\begin{align*}
		\bE{\eta^2} &= \bE\left\|\frac{1}{N}\sum_{j=1}^NY_j - \bE{Y_j}\right\|_F^2 =\frac{1}{N^2}\sum_{j=1}^N\bE||Y_j - EY_j||^2_F 
		= \frac{1}{N^2}\sum_{j=1}^N(\bE||Y_j||^2_F - ||EY_j||^2_F) \leq \frac{3m^2\kappa^4}{N}.
	\end{align*}
}
	By Chebyshev inequality, for any $t > 0$,
		$\bP(\eta \geq t) \leq \frac{\bE{\eta}}{t} \leq \frac{\sqrt{\bE{\eta^2}}}{t} \leq \frac{\sqrt{3}m\kappa^2}{\sqrt{N}t}$.
	Let $\delta = \frac{\sqrt{3}m\kappa^2}{\sqrt{N}t} \equivalent t = \frac{\sqrt{3}m\kappa^2}{\sqrt{N}\delta}$,
	then 
		$\bP\{\Wbf: ||\hat{K}_{\Wbf}[\Xbf] - K[\Xbf]||_F = \eta(\Wbf) \leq \frac{\sqrt{3}m\kappa^2}{\sqrt{N}\delta}\}\geq 1-\delta$.
	Similarly, $\bE{||\hat{K}_{\Wbf}[\Xbf]||_F} \leq \frac{1}{N}\sum_{j=1}^N\bE{||Y_j||_F} = m\kappa^2 \imply$
		$\bP\{\Wbf: ||\hat{K}_{\Wbf}[\Xbf]||_F \leq \frac{m \kappa^2}{\delta} \}\geq 1-\delta$.
	Computing the intersection of these two events and replacing $\delta$ by $\frac{\delta}{2}$, we 
	obtain the desired bounds.
\end{proof}

\begin{proof}
	[\textbf{Proof of Theorem \ref{theorem:Sinkhorn-estimate-unknown-1}}]
	With each pair $(\Wbf^1,\Wbf^2)$, by Theorem \ref{theorem:Sinkhorn-convergence},
	\begin{align*}
		\Delta &= \Delta(\Wbf^1,\Wbf^2) = \left|\Srm^{\ep}_2\left[\Ncal\left(0, (1/m)\hat{K}^1_{\Wbf^1}[\Xbf]\right), \Ncal\left(0, (1/m)\hat{K}^2_{\Wbf^2}[\Xbf]\right)\right]
		\right.
		\\
		&\quad \quad \quad \quad \quad \quad \quad \quad\left.- \Srm^{\ep}_2\left[\Ncal\left(0, (1/m)K^1[\Xbf]\right), \Ncal\left(0, (1/m)K^2[\Xbf]\right)\right]
		\right|
	\\
		&\leq \frac{3}{\ep m^2}[||\hat{K}^1_{\Wbf^1}[\Xbf]||_F + ||K^1[\Xbf]||_F + 2||K^2[\Xbf]||_F]||\hat{K}^1_{\Wbf^1}[\Xbf]-K^1[\Xbf]||_F
		\\
		& \quad + \frac{3}{\ep m^2}[2||\hat{K}^1_{\Wbf^1}[\Xbf]||_F + ||K^1[\Xbf]||_F + ||K^2[\Xbf]||_F]||\hat{K}^2_{\Wbf^2}[\Xbf]-K^2[\Xbf]||_F.
	\end{align*}
	For $0 < \delta < 1$,
		by Proposition \ref{proposition:concentration-empirical-covariance},
	the following sets satisfy $P_i^N(U_i) \geq 1- \frac{\delta}{2}$, $i=1,2$,
	{\small
	\begin{align*}
			&U_i = 	\left\{\Wbf^i \in \Omega_i^N: ||\hat{K}^i_{\Wbf^i}[\Xbf] - K^i[\Xbf]||_F \leq \frac{4\sqrt{3}m\kappa_i^2}{\sqrt{N}\delta},||\hat{K}^i_{\Wbf^i}[\Xbf]||_F \leq \frac{4m\kappa_i^2}{\delta}\right\}.
	\end{align*}
}
	Let $U = (U_1 \times \Omega_2) \cap (\Omega_1 \times U_2)$,
	then $(P_1\otimes P_2)^N(U) \geq 1 -\delta$ and $\forall (\Wbf^1, \Wbf^2) \in U$,
	\begin{align*}
		\Delta(\Wbf^1,\Wbf^2) 
		&\leq \frac{3}{\ep m^2}\left[ \frac{4m\kappa_1^2}{\delta} + m\kappa_1^2 + 2m\kappa_2^2\right]
		\frac{4\sqrt{3}m\kappa_1^2}{\sqrt{N}\delta}
		+ \frac{3}{\ep m^2}\left[ \frac{8m\kappa_1^2}{\delta} + m\kappa_1^2 + m\kappa_2^2\right]
		\frac{4\sqrt{3}m\kappa_2^2}{\sqrt{N}\delta}
		\\
		& = \frac{12\sqrt{3}}{\ep \delta}\left[\left(1+\frac{4}{\delta}\right)\kappa_1^4 + \left(3 + \frac{8}{\delta}\right)\kappa_1^2\kappa_2^2 + \kappa_2^4\right]\frac{1}{\sqrt{N}}.
	\end{align*}
	
\end{proof}

\begin{proof}
	[\textbf{Proof of Theorem \ref{theorem:Sinkhorn-estimate-unknown-2}}]
	For each fixed $\Xbf \in (T,\nu)^m$, define
	\begin{align*}
		\Delta_1 = 
		&\left|\Srm^{\ep}_{2}\left[\Ncal\left(0, (1/m)K^1[\Xbf]\right), \Ncal\left(0, (1/m)K^2[\Xbf]\right) - \Srm^{\ep}_{2}\left[\Ncal(0,C_{K^1}), \Ncal(0,C_{K^2} )\right]\right]\right|. 
	\end{align*}
	By Theorem \ref{theorem:Sinkhorn-finite-sample-GaussianProcess}, the following set $U_1 \subset (T,\nu)^m$ satisfies $\nu^m(U_1) \geq 1-\frac{\delta}{2}$
	{\small
	\begin{align*}
		U_1 = \left\{\Xbf \in (T,\nu)^m: \Delta_1 \leq \frac{6}{\ep}(\kappa_1^2 + \kappa_2^2)^2\left[ \frac{2\log\frac{12}{\delta}}{m} + \sqrt{\frac{2\log\frac{12}{\delta}}{m}}\right]\right\}.
	\end{align*}
}
	For each fixed $\Xbf\in (T,\nu)^m, \Wbf^1 \in (\Omega_1,P_1)^N, \Wbf^2 \in (\Omega_2,P_2)^N$, define
	{\small
	\begin{align*}
		\Delta_2 &= \left|\Srm^{\ep}_2\left[\Ncal\left(0, (1/m)\hat{K}^1_{\Wbf^1}[\Xbf]\right), \Ncal\left(0, (1/m)\hat{K}^2_{\Wbf^2}[\Xbf]\right)\right]
		\right.
		\nonumber
		\\
		&\quad \quad\left.
		- \Srm^{\ep}_2\left[\Ncal\left(0, (1/m)K^1[\Xbf]\right), \Ncal\left(0, (1/m)K^2[\Xbf]\right)\right]
		\right|.
	\end{align*}
}
	By Theorem \ref{theorem:Sinkhorn-estimate-unknown-1}, the following set $U_2\in (\Omega_1,P_1)^N \times (\Omega_2, P_2)^N$ satisfies $(P_1 \otimes P_2)^N(U_2) \geq 1-\frac{\delta}{2}$
	{\small
	\begin{align*}
		U_2 = \left\{(\Wbf^1, \Wbf^2): \Delta_2 \leq
		\frac{24\sqrt{3}}{\ep \delta}\left[\left(1+\frac{8}{\delta}\right)\kappa_1^4 + \left(3 + \frac{16}{\delta}\right)\kappa_1^2\kappa_2^2 + \kappa_2^4\right]\frac{1}{\sqrt{N}} \right\}.
	\end{align*}
}
	Let $U = (U_1 \times (\Omega_1,P_1)^N \times (\Omega_2,P_2)^N)\cap ((T,\nu)^m \times U_2)$,
	then $(\nu^m \otimes P_1^N \otimes P_2^N)(U) \geq 1-\delta$ and
	\begin{align*}
		\Delta_1 + \Delta_2 &\leq
		\frac{6}{\ep}(\kappa_1^2 + \kappa_2^2)^2\left[ \frac{2\log\frac{12}{\delta}}{m} + \sqrt{\frac{2\log\frac{12}{\delta}}{m}}\right]
		\\
		&\quad +
		\frac{24\sqrt{3}}{\ep \delta}\left[\left(1+\frac{8}{\delta}\right)\kappa_1^4 + \left(3 + \frac{16}{\delta}\right)\kappa_1^2\kappa_2^2 + \kappa_2^4\right]\frac{1}{\sqrt{N}},
		\;\;\forall (\Xbf, \Wbf^1, \Wbf^2)
		\in U.
	\end{align*}
\end{proof}

\begin{lemma}
	\label{lemma:concentration-trace}
	Under Assumptions $1-5$, let $\Xbf = (x_i)_{i=1}^m$ be independently sampled from $(T,\nu)$.
	For any $0 < \delta < 1$, with probability at least $1-\delta$,
	\begin{align}
		|\trace(L_{K,\Xbf}) - \trace(L_K)| \leq \kappa^2\left(\frac{2\log\frac{2}{\delta}}{m} + \sqrt{\frac{2\log\frac{2}{\delta}}{m}}\right).
	\end{align}
\end{lemma}
\begin{proof}
	Define the random variable $\eta :(T,\nu) \mapto \R$ by $\eta(x) = K(x,x)$, then 
	$||\eta||_{\infty} \leq \kappa^2$ and 
		$\trace(L_{K,\Xbf}) = \frac{1}{m}\trace\left[\sum_{i=1}^m K_{x_i} \otimes K_{x_i}\right] = \frac{1}{m}\sum_{i=1}^mK(x_i,x_i) = \frac{1}{m}\sum_{i=1}^m\eta(x_i)$,
		$\trace(L_K) = \int_{T}K(x,x)d\nu(x) = \bE{\xi},\;\;
		\bE|\eta^2|  = \int_{T}K(x,x)^2d\nu(x) \leq \kappa^4$.
	The desired bound then follows from Proposition \ref{proposition:Pinelis}.
\end{proof}

\begin{lemma}
	[\text{Lemma 4.1 in \cite{Powers1970free}}]
	\label{lemma:trace-HS-square-root}
	For $A,B \in \Sym^{+}(\H) \cap \Tr(\H)$,
	\begin{align}
		||A^{1/2}-B^{1/2}||^2_{\HS} \leq ||A-B||_{\tr}.
	\end{align}
\end{lemma}

\begin{proof}
	[\textbf{Proof of Theorem \ref{theorem:Wasserstein-convergence-Gaussian-processes}}]
	Eqs.\eqref{equation:Wass-G-reprensetation} and \eqref{equation:Wass-G-representation-empirical}
	follow as in the case of the Sinkhorn divergence.
	By Lemma \ref{lemma:concentration-trace}, $\forall0 < \delta < 1$, the following sets satisfy $\nu^m(U_i) \geq 1-\frac{\delta}{3}$, $i=1,2$,
	{\small 
	\begin{align*}
		U_i &= \left\{\Xbf \in (T,\nu)^m: |\trace(L_{K^i,\Xbf}) - \trace(L_{K^i})| \leq 
		\kappa_i^2\left[\frac{2\log\frac{6}{\delta}}{m} + \sqrt{\frac{2\log\frac{6}{\delta}}{m}}\right]\right\}.
	\end{align*}
}
	Under the assumption $\dim(\H_{K^2}) < \infty$, we have by Lemma \ref{lemma:trace-HS-square-root},
	\begin{align*}
		&|\trace[(R_{12,\Xbf}^{*}R_{12,\Xbf})^{1/2}] - \trace[(R_{12}^{*}R_{12})^{1/2}]|
		\leq ||(R_{12,\Xbf}^{*}R_{12,\Xbf})^{1/2} - (R_{12}^{*}R_{12})^{1/2}||_{\tr(\H_{K^2})}
		\\
		& \leq\sqrt{\dim(\H_{K^2})} ||(R_{12,\Xbf}^{*}R_{12,\Xbf})^{1/2} - (R_{12}^{*}R_{12})^{1/2}||_{\HS(\H_{K^2})} 
		\\
		&\leq \sqrt{\dim(\H_{K^2})}\sqrt{||(R_{12,\Xbf}^{*}R_{12,\Xbf}) - (R_{12}^{*}R_{12})||_{\tr(\H_{K^2})}}.
	\end{align*}
	By Proposition \ref{propopsition:Sinkhorn-cross-covariance}, the following set satisfies
	$\nu^m(U_3) \geq 1-\frac{\delta}{3}$,
	{\small
	\begin{align*}
		U_3 = \left\{\Xbf \in (T,\nu)^m: ||R_{12,\Xbf}^{*}R_{12,\Xbf} - R_{12}^{*}R_{12}||_{\tr(\H_{K^2})}
		\leq 2\kappa_1^2\kappa_2^2\left[ \frac{2\log\frac{6}{\delta}}{m} + \sqrt{\frac{2\log\frac{6}{\delta}}{m}}\right]\right\}.
	\end{align*}
} 
	Let $U = U_1 \cap U_2 \cap U_3$. As in the proof of Theorem \ref{theorem:Sinkhorn-finite-sample-GaussianProcess},
	$\nu^m(U) \geq 1-\delta$ and $\forall \Xbf \in U$,
	\begin{align*}
		&	\left|
		W^2_2\left[\Ncal\left(0, (1/m)K^1[\Xbf]\right), \Ncal\left(0, (1/m)K^2[\Xbf]\right) - W^2_2[\Ncal(0,C_{K^1}), \Ncal(0, C_{K^2})]\right]
		\right|
		\\
		&
		\leq (\kappa_1^2 + \kappa_2^2)\left[ \frac{2\log\frac{6}{\delta}}{m} + \sqrt{\frac{2\log\frac{6}{\delta}}{m}}\right]
		+ 2\sqrt{2}\kappa_1\kappa_2\sqrt{\dim(\H_{K^2})}\sqrt{ \frac{2\log\frac{6}{\delta}}{m} + \sqrt{\frac{2\log\frac{6}{\delta}}{m}}}.
	\end{align*}
\end{proof}

The following is a special case of Corollary 4 in \cite{Minh:2021EntropicConvergenceGaussianMeasures}, where
$\mu_{\Xbf}$ and $C_{\Xbf}$ are the sample mean and sample covariance matrix, respectively, based on the sample $\Xbf = (x_i)_{i=1}^N$.

\begin{proposition}
	[\textbf{Estimation of 2-Wasserstein distance between Gaussian measures on $\R^d$}]
	\label{proposition:sample-complexity-Wasserstein-GaussianRd}
	Let $\rho_i = \Ncal(\mu_i,C_i)$ on $\R^d$, $i=1,2$.
	Let $\Xbf = (x_i)_{i=1}^N$ and $\Ybf = (y_j)_{j=1}^N$ be independently sampled from $(\R^d,\rho_1)$
	and $(\R^d, \rho_2)$, respectively.
	$\forall 0 < \delta < 1$, with probability at least $1-\delta$,
	\begin{align}
		&\left|W_2[\Ncal(\mu_{\Xbf}, C_{\Xbf}), \Ncal(\mu_{\Ybf}, C_{\Ybf})] - W_2(\Ncal(\mu_1,C_1), \Ncal(\mu_2,C_2))\right|
		\nonumber
		\\
		%
		& 
		\quad \leq 
		2(\eta_1 + \eta_2)\sqrt{\frac{4}{N\delta^2} + \frac{\sqrt{d}}{\sqrt{N}\delta}\left(3 + \frac{4}{\sqrt{N}\delta}\right)},
	\end{align}
	where 
	$\eta_i = (2||C_i||^2_{\HS} + 4\la \mu_i, C_i\mu_i\ra + (\trace{C_i} +||\mu_i||^2)^2)^{1/4}$, $i=1,2$.
\end{proposition}

\begin{proof}
	[\textbf{Proof of Theorem \ref{theorem:Wass-estimate-unknown-1}}]
	Apply Proposition \ref{proposition:sample-complexity-Wasserstein-GaussianRd} with
	$d=m$, $\mu_i = 0$,  $C_i = \frac{1}{m}K^i[\Xbf]$, and
	$\trace(C_i) = \frac{1}{m}\trace(K^i[\Xbf]) = \frac{1}{m}\sum_{j=1}^mK^i(x_j,x_j) \leq \kappa_i^2$, $||C_i||^2_{\HS} \leq [\trace(C_i)]^2 \leq \kappa_i^4$. 
	Thus $\eta_i = (2||C_i||^2_{\HS} + [\trace(C_i)]^2)^{1/4} \leq 3^{1/4}\kappa_i < \frac{3}{2}\kappa_i$.
\end{proof}

\begin{proof}
	[\textbf{Proof of Theorem \ref{theorem:Wass-estimate-unknown-2}}]
	This follows by combining Theorems \ref{theorem:Wasserstein-convergence-Gaussian-processes} and \ref{theorem:Wass-estimate-unknown-1}, as in Theorem \ref{theorem:Sinkhorn-estimate-unknown-2}. Here we make use of the elementary inequality $|a-b|^2 \leq |a^2 - b^2|$ for $a\geq 0, b \geq 0$.
\end{proof}


\begin{lemma}
	\label{lemma:Fubini-switch}
	Under Assumptions 1-4, $\forall x \in T, \forall f \in \Lcal^2(T,\nu)$,
	\begin{align}
		&\bE[\xi \otimes \xi]f(x) = \bE\int_{T}\xi(\omega, x)\xi(\omega, t)f(t)d\nu(t) = \int_{T}K(x,t)f(t)d\nu(t).
	\end{align}
\end{lemma}
\begin{proof}
	By H\"older Theorem and Tonelli Theorem, $\forall f \in \Lcal^2(T,\nu)$, $\forall x \in T$,
	\begin{align*}
		&\left(\int_{\Omega}\int_{T}|\xi(\omega, x)\xi(\omega,t)f(t)|d\nu(t)dP(\omega)\right)^2 
		\\
		&\leq 
		\int_{\Omega \times T}\xi(\omega,t)^2d\nu(t)dP(\omega)\int_{\Omega \times T}\xi(\omega, x)^2f(t)^2dP(\omega)d\nu(t)
		\\
		& = ||f||^2_{\Lcal^2(T,\nu)}K(x,x)\int_{T}K(t,t)d\nu(t) \leq \kappa^2||f||^2_{\Lcal^2(T,\nu)}K(x,x) < \infty.
	\end{align*}
	Thus $\xi(., x)\xi(.,.)f \in \Lcal^1(\Omega \times T, P \times \nu)$. By Fubini Theorem,
	\begin{align*}
		&\bE\int_{T}\xi(\omega, x)\xi(\omega, t)f(t)d\nu(t) = \int_{\Omega}\left(\int_{T}\xi(\omega, x)\xi(\omega,t)f(t)d\nu(t)\right)dP(\omega) 
		\\
		&= \int_{T}\left(\int_{\Omega}\xi(\omega, x)\xi(\omega,t)f(t)dP(\omega)\right)d\nu(t)
		= \int_{T}K(x,t)f(t)d\nu(t).
	\end{align*}
\end{proof}

\begin{proof}
	[\textbf{Proof of Lemma \ref{lemma:sample-covariance-expectation}}]
	By definition of $K_{\Wbf}$, Lemma \ref{lemma:fourth-moment-Gaussian-process}, and Tonelli Theorem,
	\begin{align*}
		&\bE\int_{T}K_{\Wbf}(x,x)d\nu(x) = \int_{T}\bE[\xi(\omega,x)^2]d\nu(x) = \int_{T}K(x,x)d\nu(x) \leq \kappa^2,
		\\
		&\int_{T}K_{\Wbf}(x,x)^2d\nu(x)=\frac{1}{N^2}\int_{T}\left(\sum_{i=1}^N \xi(\omega_i,x)^2\right)^2d\nu(x)
		\leq \frac{1}{N}\int_{T}\sum_{i=1}^N\xi(\omega_i,x)^4d\nu(x),
		\\
		&\bE\int_{T}K_{\Wbf}(x,x)^2d\nu(x) \leq \int_{T}\bE[\xi(\omega, x)^4]d\nu(x) = 3\int_{T}K(x,x)^2d\nu(x)
		\leq 3\kappa^4.
	\end{align*}
\end{proof}
\begin{proof}	
	[\textbf{Proof of Proposition \ref{proposition-concentration-sample-cov-operator}}]
	
	Define random variable $Y_j: (\Omega, P)^N \mapto \HS(\Lcal^2(T,\nu))$ by $Y_j(\Wbf) = \xi(\omega_j,.) \otimes \xi(\omega_j,.)$, where $\Wbf =(\omega_1, \ldots, \omega_N)$ is independently sampled from $\Omega$.
	Then the $Y_j$'s are IID and $C_K = \bE{Y_j}$, $C_{K,\Wbf} = \frac{1}{N}\sum_{j=1}^NY_j(\Wbf)$, and
	\begin{align*}
		||Y_j||_{\HS(\Lcal^2(T,\nu))} &= ||\xi(\omega_j,.)||^2_{\Lcal^2(T,\nu)} = \int_{T}\xi(\omega_j,t)^2d\nu(t),
		\\
		\bE||Y_j||_{\HS(\Lcal^2(T,\nu))} &= \int_{\Omega}\int_{T}\xi(\omega,t)^2d\nu(t)dP(\omega) = \int_{T}K(t,t)d\nu(t) \leq \kappa^2,
		\\
		\bE||Y_j||^2_{\HS(\Lcal^2(T,\nu))} & = 
		\int_{\Omega}\left(\int_{T}\xi(\omega,t)^2d\nu(t)\right)^2dP(\omega)= \bE||\xi||^4_{\Lcal^2(T,\nu)}
		\\
		&\leq \int_{\Omega}\int_{T}\xi(\omega,t)^4d\nu(t)dP(\omega) = 3\int_{T}K(t,t)^2d\nu(t) \leq 3 \kappa^4
		\;\;\text{by Lemma \ref{lemma:fourth-moment-Gaussian-process}}.
	\end{align*}
	Define the random variable $\eta: (\Omega, P)^N \mapto \R$ by
	$\eta(\Wbf) = \left\|\frac{1}{N}Y_j(\Wbf) - \bE{Y_j}\right\|^2_{\HS} = ||C_{K,\Wbf} - C_K||^2_{\HS}$.
	Since the $Y_j$'s are independent, identically distributed,
	{\small
		\begin{align*}
			\bE{\eta^2} &= \bE\left\|\frac{1}{N}\sum_{j=1}^NY_j - \bE{Y_j}\right\|_{\HS}^2 =\frac{1}{N^2}\sum_{j=1}^N\bE||Y_j - EY_j||^2_{\HS} 
			= \frac{1}{N^2}\sum_{j=1}^N(\bE||Y_j||^2_{\HS} - ||EY_j||^2_{\HS}) \leq \frac{3\kappa^4}{N}.
		\end{align*}
	}
	By the Chebyshev inequality, for any $t > 0$,
	$	\bP(\eta \geq t) \leq \frac{\bE{\eta}}{t} \leq \frac{\sqrt{\bE{\eta^2}}}{t} \leq \frac{\sqrt{3}\kappa^2}{\sqrt{N}t}$.
	Let $\delta = \frac{\sqrt{3}\kappa^2}{\sqrt{N}t} \equivalent t = \frac{\sqrt{3}\kappa^2}{\sqrt{N}\delta}$,
	then 
	$\bP\{\Wbf: ||C_{K,\Wbf} - C_K||_{\HS} = \eta(\Wbf) \leq \frac{\sqrt{3}\kappa^2}{\sqrt{N}\delta}\} \geq 1-\delta$.
	Similarly, since $\bE{||C_{K,\Wbf}||_{\HS}} \leq \frac{1}{N}\sum_{j=1}^N\bE{||Y_j||_F} = \kappa^2$, we have
	$\bP\{\Wbf: ||C_{K,\Wbf}||_{\HS} \leq \frac{\kappa^2}{\delta}\} \geq 1-\delta$.
	Computing the intersection of these two sets, replacing $\delta$ with $\frac{\delta}{2}$, gives us the desired result.
\end{proof}

\begin{proof}
	[\textbf{Proof of Theorem \ref{theorem:Sinkhorn-convergence-sample-covariance-operator}}]
	By Theorem \ref{theorem:Sinkhorn-convergence},
	\begin{align*}
		&\Delta = \left|\Srm^{\ep}_2[\Ncal(0,C_{K^1,\Wbf^1}), \Ncal(0,C_{K^2,\Wbf^2})] - \Srm^{\ep}_2[\Ncal(0,C_{K^1}),\Ncal(0,C_{K^2})]\right|
		\nonumber
		\\
		&\leq \frac{3}{\ep}[||C_{K^1,\Wbf^1}||_{\HS} + ||C_{K^1}||_{\HS} + 2||C_{K^2}||_{\HS}]||C_{K^1,\Wbf^1} - C_{K^1}||_{\HS}
		\nonumber
		\\
		&\quad +\frac{3}{\ep}[2||C_{K^1,\Wbf^1}||_{\HS} + ||C_{K^1}||_{\HS} + ||C_{K^2}||_{\HS}]||C_{K^2,\Wbf^2} - C_{K^2}||_{\HS}.
	\end{align*}
	For $0 < \delta < 1$, by Proposition \ref{proposition-concentration-sample-cov-operator}, the following sets satisfy $P_i^N(U_i) \geq 1-\frac{\delta}{2}$,
	$i=1,2$,
	\begin{align*}
		U_i &= \left\{\Wbf^i \in (\Omega_i,P_i)^N:||C_{K^i,\Wbf^i}||_{\HS} \leq \frac{4\kappa_i^2}{\delta}, ||C_{K^i,\Wbf^i} - C_{K^i}||_{\HS} \leq \frac{4\sqrt{3}\kappa_i^2}{\sqrt{N}\delta} \right\}.
	\end{align*}
	Then $U = (U_1 \times \Omega_2^N) \cap (\Omega_1^N\times  U_2)$
	satisfies $(P_1^N \times P_2^N)(U) \geq 1 - \delta$. For $(\Wbf^1, \Wbf^2) \in U$,
	\begin{align*}
		\Delta &\leq \frac{3}{\ep}\left[\frac{4\kappa_1^2}{\delta} + \kappa_1^2 + 2\kappa_2^2\right]\frac{4\sqrt{3}\kappa_1^2}{\sqrt{N}\delta} + \frac{3}{\ep}\left[\frac{8\kappa_1^2}{\delta} + \kappa_1^2 + \kappa_2^2\right]\frac{4\sqrt{3}\kappa_2^2}{\sqrt{N}\delta}
		\\
		& \leq \frac{12\sqrt{3}}{\ep\sqrt{N}\delta}\left(\left(1+\frac{4}{\delta}\right)\kappa_1^4 + \left(3+\frac{8}{\delta}\right)\kappa_1^2\kappa^2 + \kappa_2^4\right).
	\end{align*}
\end{proof} 

\begin{proof}
	[\textbf{Proof of Theorem \ref{theorem:Sinkhorn-convergence-sample-covariance-operator-finite-sample}}]
	This is similar to 
	Theorem \ref{theorem:Sinkhorn-estimate-unknown-2}.
	By Lemma \ref{lemma:sample-covariance-expectation},
	$\forall 0 < \delta < 1$, 
	\begin{align*}
		\bP\left\{(\Wbf^1,\Wbf^2) \in (\Omega_1,P_1)^N \times (\Omega_2,P_2)^N:\int_{T}K_{\Wbf^1}(x,x)^2d\nu(x) \leq \frac{24\kappa_1^4}{\delta}, 
		\right.
		\\
		\left.
		\int_{T}K_{\Wbf^2}(x,x)^2d\nu(x) \leq \frac{24\kappa_2^4}{\delta}
		\right\} \geq 1-\frac{\delta}{4}.
	\end{align*}
	Let
	$	\Delta_1 =\left|\Srm^{\ep}_2[\Ncal(0,C_{K^1,\Wbf^1}), \Ncal(0,C_{K^2,\Wbf^2})] - \Srm^{\ep}_2[\Ncal(0,C_{K^1}),\Ncal(0,C_{K^2})]\right|
	$. By Theorem \ref{theorem:Sinkhorn-convergence-sample-covariance-operator},
	\begin{align*}
		&\bP\left\{(\Wbf^1,\Wbf^2) \in (\Omega_1,P_1)^N\times (\Omega_2,P_2)^N: \right.
		\\
		&\quad \left.\Delta_1 \leq \frac{48\sqrt{3}}{\ep\delta}\left(\left(1+\frac{16}{\delta}\right)\kappa_1^4 + \left(3+\frac{32}{\delta}\right)\kappa_1^2\kappa^2 + \kappa_2^4\right)\frac{1}{\sqrt{N}}  \right\}
		\geq 1 -\frac{\delta}{4}.
	\end{align*}
	It follows that the following set satisfies $\bP(U_1) \geq 1 -\frac{\delta}{2}$,
	\begin{align*}
		U_1 &=\left\{(\Wbf^1,\Wbf^2) \in (\Omega_1,P_1)^N\times (\Omega_2,P_2)^N:
		\int_{T}K_{\Wbf^1}(x,x)^2d\nu(x) \leq \frac{24\kappa_1^4}{\delta},\right.
		\\ 
		&\left.
		\int_{T}K_{\Wbf^2}(x,x)^2d\nu(x) \leq \frac{24\kappa_2^4}{\delta},
		\Delta_1 \leq \frac{48\sqrt{3}}{\ep\delta}\left(\left(1+\frac{16}{\delta}\right)\kappa_1^4 + \left(3+\frac{32}{\delta}\right)\kappa_1^2\kappa^2 + \kappa_2^4\right)\frac{1}{\sqrt{N}}  \right\}.
	\end{align*}
	Let $\Delta_2 = \left|\Srm^{\ep}_2[\Ncal(0, \frac{1}{m}\hat{K}^1_{\Wbf^1}[\Xbf]), \Ncal(0, \frac{1}{m}\hat{K}^2_{\Wbf^2}[\Xbf])]- \Srm^{\ep}_2[\Ncal(0,C_{K^1,\Wbf^1}), \Ncal(0,C_{K^2,\Wbf^2})]\right|$.
	Theorem \ref{theorem:Sinkhorn-finite-sample-GaussianProcess-unbounded} implies that
	for $(\Wbf^1, \Wbf^2) \in U_1$ fixed, the following set satisfies $\bP(U_2) \geq 1-\frac{\delta}{2}$,
	\begin{align*}
		U_2 = \left\{\Xbf \in (T,\nu)^m: \Delta_2 \leq \frac{18}{\ep}(((24/\delta)^{1/4}\kappa_1)^2+((24/\delta)^{1/4}\kappa_2)^2)^2\left(1+\frac{12}{\delta}\right)\frac{2}{\sqrt{m}\delta}\right.
		\\
		= \left.\frac{864}{\ep\delta}(\kappa_1^2 + \kappa_2^2)^2\left(1+\frac{12}{\delta}\right)\frac{1}{\sqrt{m}}\right\}.
	\end{align*}
	Let $U = (U_1 \times (T,\nu)^m) \cap (((\Omega_1,P_1)^N\times (\Omega_2,P_2)^N) \times U_2)$,
	then $\bP(U) \geq 1-\delta$.
\end{proof}

\section{Estimation of Hilbert-Schmidt distance}
\label{section:Hilbert-Schmidt-distance}

For completeness, we present the finite sample estimate of 
$||C_{K^1} - C_{K^2}||_{\HS(\Lcal^2(T,\nu))}$.
For this, it is {\it not}  necessary to assume that $C_{K^i}$, $i=1,2$ are self-adjoint, positive.
The only requirement is that $C_{K^i} \in \HS(\Lcal^2(T,\nu))$.

{\bf Assumption 7}. Let $T$ be a complete, separable metric space, $\nu$ a Borel probability measure 
on $T$. Let $K, K^1, K^2: T \mapto \R$ be pointwise defined.
Assume $\exists \kappa, \kappa_1,\kappa_2 > 0$ such that
\begin{align}
	\sup_{x,y \in T\times T}|K(x,y)| \leq \kappa^2, \; \sup_{(x,y) \in T\times T}|K^i(x,y)| \leq \kappa_i^2, \;\;i=1,2.
\end{align}
It is well-known (see e.g. Theorem VI.23 in \cite{ReedSimon:Functional}) that the following operator $C_K:\Lcal^2(T,\nu) \mapto \Lcal^2(T,\nu)$ (similarly $C_{K^i}$, $i=1,2$) is Hilbert-Schmidt
\begin{align}
	(C_Kf)(x) &= \int_{T}K(x,y)f(y)d\nu(y), \text{ with }
	||C_K||^2_{\HS(\Lcal^2(T,\nu))} = \int_{T \times T}K(x,y)^2d\nu(x)d\nu(y).
\end{align}

\begin{theorem}
	[\textbf{Estimation of Hilbert-Schmidt distance}]
	\label{theorem:HS-convergence-general-HS-operators}
	Under Assumption 7,
	let $\Xbf = (x_i)_{i=1}^m$, $\Ybf = (y_j)_{j=1}^n$ be independently sampled from $(T, \nu)$.
	Let $K[\Xbf,\Ybf] \in \R^{m \times n}$ 
	be defined by $(K[\Xbf, \Ybf])_{ij} = K(x_i, y_j)$, 
	$1 \leq i \leq m$, $1\leq j \leq n$.
	$\forall 0 < \delta < 1$, with probability at least $1-\delta$,
	\begin{align}
		&\left|\frac{1}{mn}||K^1[\Xbf,\Ybf] - K^2[\Xbf,\Ybf]||_F^2 - ||C_{K^1} - C_{K^2}||^2_{\HS(\Lcal^2(T,\nu))}\right|
		\nonumber
		\\
		&
		\leq (\kappa_1^2 + \kappa_2^2)^2 \left(\frac{2\log\frac{2}{\delta}}{mn} + \sqrt{\frac{2\log\frac{2}{\delta}}{mn}}\right).
	\end{align}
\end{theorem}




\begin{lemma}
	\label{lemma:HS-squarenorm-estimage-general}
	Under Assumption 7,
	let $\Xbf = (x_i)_{i=1}^m$, $\Ybf = (y_j)_{j=1}^n$ be independently sampled from $(T, \nu)$.
	$\forall 0 < \delta < 1$, with probability at least $1-\delta$,
	\begin{align}
		\left|\frac{1}{mn}||K[\Xbf,\Ybf]||_F^2 - ||C_K||^2_{\HS(\Lcal^2(T,\nu))}\right|
		\leq \kappa^4 \left(\frac{2\log\frac{2}{\delta}}{mn} + \sqrt{\frac{2\log\frac{2}{\delta}}{mn}}\right).
	\end{align}
\end{lemma}
\begin{proof}
	Define the random variable $\eta: (T\times T, \nu \otimes \nu) \mapto \R$ by $\eta(x,y) = K(x,y)^2$.
	Then 
	\begin{align*}
		||\eta||_{\infty} \leq \kappa^4, \;\;
		\bE_{x,y}{\eta} &= \int_{T\times T}K(x,y)^2d\nu(x)d\nu(y) = ||C_K||^2_{\HS(\Lcal^2(T,\nu))},
		\\
		\frac{1}{mn}\sum_{i=1}^m\sum_{j=1}^n\eta(x_i,y_j) &= \frac{1}{mn}\sum_{i=1}^m\sum_{j=1}^nK(x_i,y_j)^2 = \frac{1}{mn}||K[\Xbf,\Ybf]||^2_F.
	\end{align*}
	The desired bound the follows from Proposition  \ref{proposition:concentration-TK2K1-empirical}.
\end{proof}

\begin{proof}
	[\textbf{Proof of Theorem \ref{theorem:HS-convergence-general-HS-operators}}]
	We apply Lemma \ref{lemma:HS-squarenorm-estimage-general} to $K^1-K^2$, noting that
	$C_{K^1} - C_{K^2} = C_{K^1-K^2}$ and $||K^1-K^2||_{\infty} \leq ||K^1||_{\infty} + ||K^2||_{\infty}
	\leq \kappa_1^2 + \kappa_2^2$.
\end{proof}

\section{Hilbert-Schmidt operators between two Hilbert spaces}
\label{section:Hilbert-Schmidt}

For completeness, we include here several properties of the set of Hilbert-Schmidt operators
$\HS(\H_1,\H_2)$
between two separable Hilbert spaces $\H_1$ and $\H_2$. Many standard texts in functional analysis
consider the set $\HS(\H)$, with $\H=\H_1 = \H_2$. The definition of $\HS(\H_1,\H_2)$ that we use here is from \cite{Kadison:1983}
\begin{align}
	\HS(\H_1, \H_2) = \{ A \in \Lcal(\H_1, \H_2): ||A||_{\HS(\H_1, \H_2)} < \infty\},
\end{align}
where the Hilbert-Schmidt norm $||\;||_{\HS(\H_1,\H_2)}$ is defined by
\begin{align}
	||A||^2_{\HS(\H_1,\H_2)} = \sum_{k,j=1}^{\infty}\la Ae_{k,1}, e_{j,2}\ra^2_{\H_2} = \sum_{k=1}^{\infty}||Ae_{k,1}||^2_{\H_2} = \trace(A^{*}A)
\end{align}
for any orthonormal bases $\{e_{k,i}\}_{k\in \Nbb}$ of $\H_i$, $i=1,2$,
independently of the choice of bases.
\begin{proposition}
	\label{proposition:trace-norm-product-HS-different-spaces}
	Let $\H_1, \H_2$ be two separable Hilbert spaces. Let
	$A,B\in \HS(\H_1,\H_2)$. Then $A^{*}B \in \Tr(\H_1)$ and
	$||A^{*}B||_{\tr(\H_1)} \leq ||A||_{\HS(\H_1,\H_2)}||B||_{\HS(\H_1, \H_2)}$.
\end{proposition}
\begin{proof}
	Consider the polar decomposition $A^{*}B = U|A^{*}B|$ where $U$ is a partial isometry on $\H_1$, i.e. an isometry 
	on the closed subspace $\H_{U} = \ker(U)^{\perp} = \ker(AB)^{\perp}$. 
	Let $\{e_k\}_{k \in \Nbb}$ be an orthonormal basis in $\H_U$, then
	$\{Ue_k\}_{k\in \Nbb}$ is also an orthonormal basis in $\H_U$ and
	\begin{align*}
		\trace|A^{*}B| &= \trace|A^{*}B|_{\H_U} = \sum_{k=1}^{\infty}\la |A^{*}B|e_{k}, e_{k}\ra_{\H_1}
		= \sum_{k=1}^{\infty}\la U|A^{*}B|e_k, Ue_k\ra_{\H_1} 
		\\
		& = \sum_{k=1}^{\infty}\la A^{*}Be_k, Ue_k\ra_{\H_1} = \sum_{k=1}^{\infty}\la Be_k, AUe_k\ra_{\H_2}
		\leq \sum_{k=1}^{\infty}||Be_k||_{\H_2}||AUe_k||_{\H_2}
		\\
		&\leq \left(\sum_{k=1}^{\infty}||Be_k||^2_{\H_2}\right)\left(\sum_{k=1}^{\infty}||AUe_k||^2_{\H_2}\right)^{1/2}
		\leq ||B||_{\HS(\H_1,\H_2)}||A||_{\HS(\H_1,\H_2)},
	\end{align*}
	where the last inequality is an equality if $\ker(AB)^{\perp} = \{0\}$, i.e. $\H_U = \H_1$.
\end{proof}

\begin{corollary}
	\label{corollary:trace-norm-AB-different-spaces}
	Let $\H_1, \H_2$ be two separable Hilbert spaces. Let
	$A,B\in \HS(\H_1,\H_2)$. Then $A^{*}A, B^{*}B \in \Tr(\H_1)$ and
	\begin{align}
		||A^{*}A - B^{*}B||_{\tr(\H_1)} \leq (||A||_{\HS(\H_1,\H_2)}+||B||_{\HS(\H_1,\H_2)})||A-B||_{\HS(\H_1,\H_2)}.
	\end{align}
\end{corollary}

\begin{lemma}
	\label{lemma:HS-norm-rank-one-product}
	Let $\H_1,\H_2$ be two separable Hilbert spaces. Let $u_1\in \H_1, u_2 \in \H_2$.
	Then $u_1 \otimes u_2 \in \HS(\H_2, \H_1)$ and
	$||u_1 \otimes u_2||_{\HS(\H_2, \H_1)} = ||u_1||_{\H_1}||u_2||_{\H_2}$.
\end{lemma}
\begin{proof} 
	Let $\{e_k\}_{k\in \Nbb}$ be any orthonormal basis in $\H_2$. By definition,
	$	||u_1 \otimes u_2||^2_{\HS(\H_2, \H_1)} = \sum_{k=1}^{\infty}||u_1\la u_2, e_k\ra_{\H_2}||^2_{\H_1} = ||u_1||^2_{\H_1}\sum_{k=1}^{\infty}|\la u_2, e_k\ra_{\H_2}|^2$
	%
	$=||u_1||_{\H_1}^2||u_2||_{\H_2}^2 < \infty
	$.
\end{proof}

\bibliographystyle{siamplain}
\bibliography{/Users/Minhs/Dropbox/cite_RKHS}

\end{document}